\documentclass{article}

\usepackage[utf8]{inputenc} 
\usepackage[T1]{fontenc}    
\usepackage{hyperref}       
\hypersetup{
    colorlinks,
    linkcolor={red!50!black},
    citecolor={blue!50!black},
    urlcolor={blue!80!black}
}
\usepackage{url}            
\usepackage{booktabs}       
\usepackage{amsfonts}       
\usepackage{nicefrac}       
\usepackage{microtype}      
\usepackage{xcolor}         
\usepackage[a4paper,
            textheight=690pt,
            textwidth=487.8225pt,
            headheight=10.0pt]{geometry}

\usepackage[mode=buildnew]{standalone}
\usepackage{svg}
\usepackage{graphicx}
\usepackage{subcaption}
\usepackage{makecell,multirow,multicol,array}
\usepackage{wrapfig}
\usepackage{eqparbox}
\usepackage{authblk}
\usepackage{amsmath,amssymb,mathtools,amsthm,bm,bbm,thmtools}
\usepackage[noabbrev]{cleveref}

\newtheorem{assumption}{Assumption}

\newcommand{\eqmathboxl}[2]{\eqmakebox[#1][l]{$\displaystyle#2$}}

\newcommand{\resultstablescaling}{0.85}

\newcommand\blfootnote[1]{%
  \begingroup
  \renewcommand\thefootnote{}\footnote{#1}%
  \addtocounter{footnote}{-1}%
  \endgroup
}

\def\assumptionsref{}

\title{UQGAN: A Unified Model for Uncertainty Quantification of Deep Classifiers trained via Conditional GANs}

\author[1]{Philipp Oberdiek\thanks{\href{mailto:philipp.oberdiek@cs.tu-dortmund.de}{philipp.oberdiek@cs.tu-dortmund.de}}}
\author[1]{Gernot A. Fink\thanks{\href{mailto:gernot.fink@cs.tu-dortmund.de}{gernot.fink@cs.tu-dortmund.de}}}
\author[2,3]{Matthias Rottmann\thanks{\href{mailto:matthias.rottmann@epfl.ch}{matthias.rottmann@epfl.ch}}}
\affil[1]{TU Dortmund University}
\affil[2]{School of Computer and Communication Sciences, EPFL}
\affil[3]{School of Mathematics and Natural Sciences, University of Wuppertal}
\affil[ ]{}

\date{}

\begin{document}

\maketitle

\blfootnote{Accepted for publication at the 36th Conference on Neural Information Processing Systems (NeurIPS 2022).}

\begin{abstract}
  We present an approach to quantifying both aleatoric and epistemic uncertainty for deep neural networks in image classification, based on generative adversarial networks (GANs). While most works in the literature that use GANs to generate out-of-distribution (OoD) examples only focus on the evaluation of OoD detection, we present a GAN based approach to learn a classifier that produces proper uncertainties for OoD examples as well as for false positives (FPs). Instead of shielding the entire in-distribution data with GAN generated OoD examples which is state-of-the-art, we shield each class separately with out-of-class examples generated by a conditional GAN and complement this with a one-vs-all image classifier. In our experiments, in particular on CIFAR10, CIFAR100 and Tiny ImageNet, we improve over the OoD detection and FP detection performance of state-of-the-art GAN-training based classifiers. Furthermore, we also find that the generated GAN examples do not significantly affect the calibration error of our classifier and result in a significant gain in model accuracy.
The code and pre-trained weights can be found here: \url{https://github.com/RonMcKay/UQGAN}
\end{abstract}

\begin{figure}[h!]
    \centering
    \includegraphics[width=0.34\textwidth,keepaspectratio]{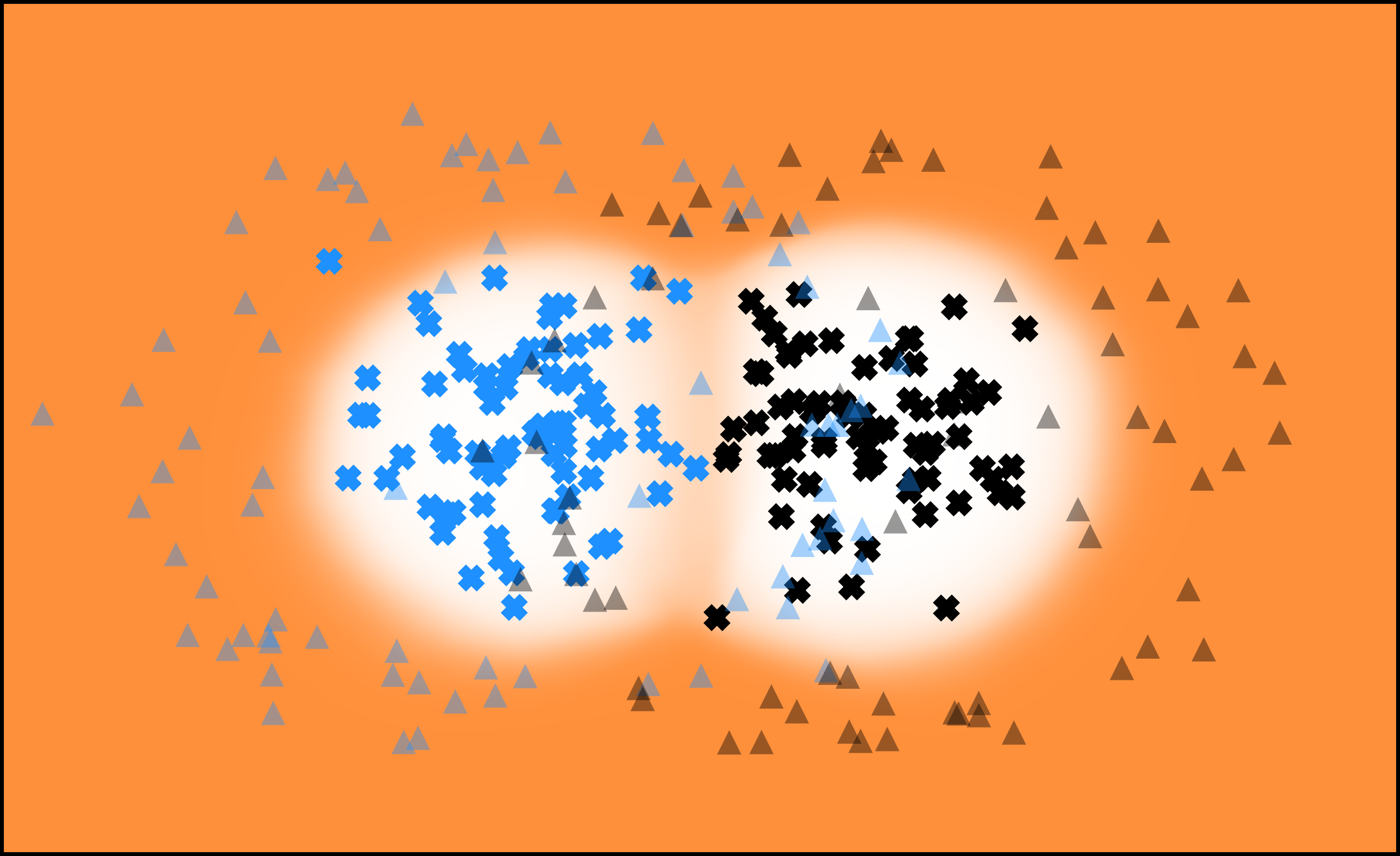}
    \hspace{0.1cm}
    \includegraphics[width=0.34\textwidth,keepaspectratio]{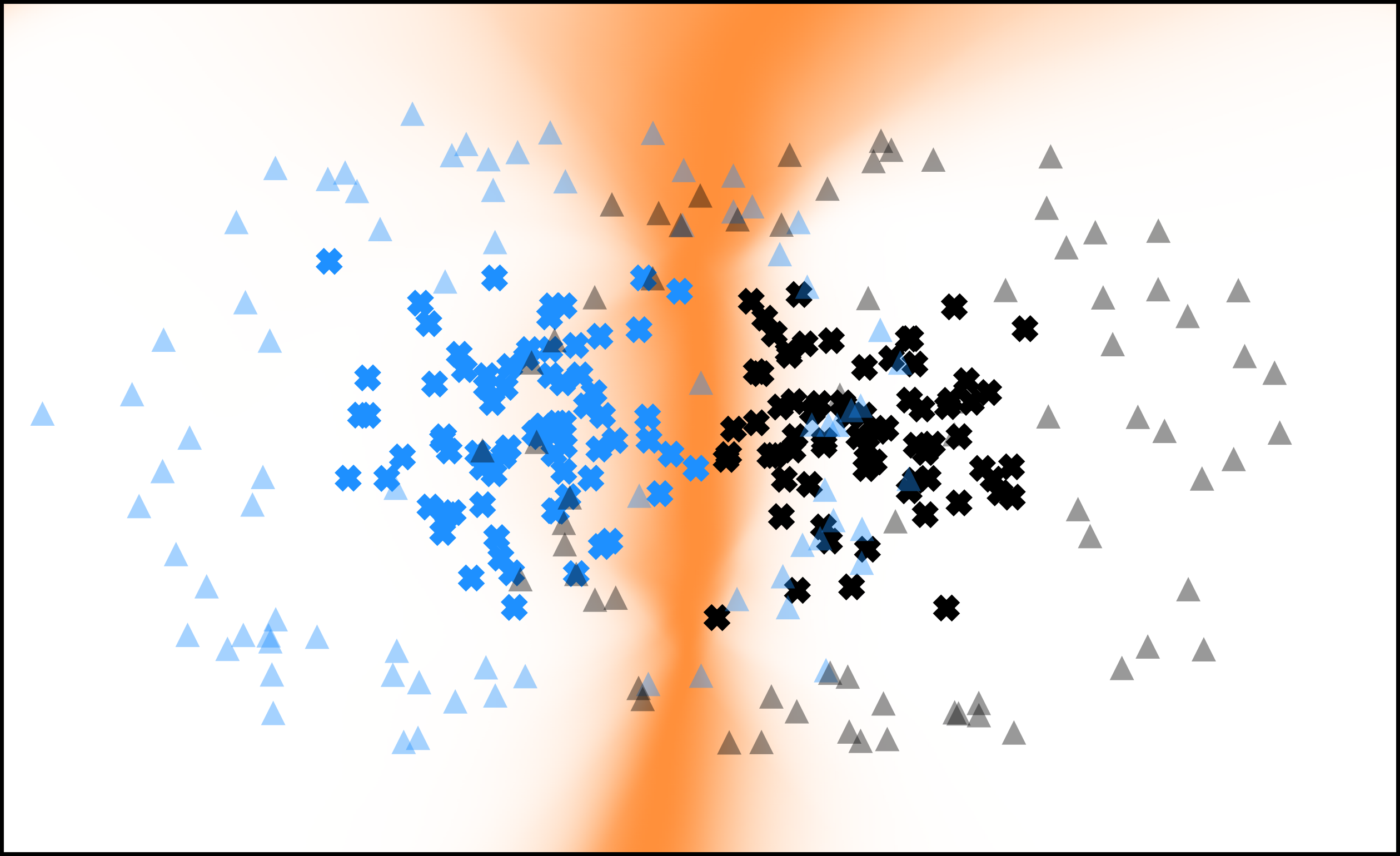}
    \caption{Toy example of two slightly overlapping Gaussian distributions. From left to right: 1. OoD heatmap with orange indicating a high probability of being OoD and white of in-distribution; 2. Aleatoric uncertainty (entropy over \Cref{eq:classifier}) with orange indicating high and white low uncertainty. Triangles indicate GAN-generated out-of-class examples and crosses correspond to the in-distribution data, while their color is coding the class membership.}
    \label{fig:toy_example}
\end{figure}

\section{Introduction}

Deep learning has shown outstanding performance in image classification tasks \cite{KrizhevskySH12,SimonyanZ14a,HeZRS16}.
However, due to the enormous capacity and the inability of rejecting examples, deep neural networks (DNNs) are not capable of expressing their uncertainties appropriately. DNNs have demonstrated the tendency to overfit training data \cite{SrivastavaHKSS14} and to be easily fooled into wrong class predictions with high confidence \cite{GoodfellowSS14,GuoPSW17}. Also, importantly, ReLU networks have been proven to show high confidence far away from training data \cite{HeinAB19}. This is contrary to the behavior that would be expected naturally by humans, which is to be uncertain when being confronted with new data examples that have not been observed during training. 

Several approaches to predictive uncertainty quantification have been introduced in recent years, considering uncertainty in a Bayesian sense \cite{BlundellCKW15,GalG15a,KendallG17} as well as from a frequentist's point of view \cite{HendrycksG17,DeVriesT18,PadhyNRLSL20}. A common evaluation protocol is to discriminate between true and false positives (FPs) by means of a given uncertainty quantification. For an introduction to uncertainty in machine learning, we refer to \cite{HullermeierW21}, for a survey on uncertainty quantification methods for DNNs see \cite{GawlikowskiTALH21}.

By design of ordinary DNNs for image classification, their uncertainty is often studied on in-distribution examples \cite{AshukhaLMV20}. The task of out-of-distribution (OoD) detection (or novelty detection) is oftentimes considered separately from uncertainty quantification \cite{LiangLS18,SnoekOFLNSDRN19,MundtPMR19}. Thus, OoD detection in deep learning has spawned an own line of research and method development. Among others, changes in architecture \cite{DeVriesT18}, loss function \cite{SensoyKK18,AmersfoortSTG20} and the incorporation of data serving as OoD proxy \cite{HendrycksMD19} have been considered in the literature. Generative adversarial networks (GANs) have been used to replace that proxy by artificially generated data examples. In \cite{LeeLLS18}, examples of OoD data are created such that they shield the in-distribution regime from the OoD regime. Note that this often requires pre-training of the generator. E.g.\ in the aforementioned work, the authors constructed OoD data in their 2D example to pretrain the generator. While showing promising results, such GAN-based approaches mostly predict a single score for OoD detection and do not yield a principled approach to uncertainty quantification distinguishing in-distribution uncertainty between classes and out-of-distribution uncertainty.

In this work, we propose to use GANs to, instead of shielding all classes at once, shield each class separately from the out-of-class (OoC) regime (also cf.\ \cref{fig:toy_example}). Instead of maximizing an uncertainty measure, like softmax entropy, we combine this with a one-vs-all classifier in the final DNN layer. This is learned jointly with a class-conditional generator for out-of-class data 
in an adversarial framework. The resulting classifiers are used to model (class conditional) likelihoods. Via Bayes rule we define posterior class probabilities in a principled way. Our work thus makes the following novel contributions:
\begin{enumerate}
    \setlength\itemsep{-0.2em}
    \item We introduce a GAN-based model yielding a classifier with complete uncertainty quantification.
    \item Our model allows to distinguish uncertainty between classes (in large sample limit, if properly learned, approaching aleatoric uncertainty) from OoD uncertainty (approaching epistemic uncertainty).
    \item By a conditional GAN trained with a Wasserstein-based loss function, we achieve class shielding in low dimensions without any pre-training of the generator.
    \item In higher dimensions, we use a class conditional autoencoder and train the GAN on the latent space. This is coherent with the conditional GAN, allows us to use less complex generators and critics and reduces the influence of adversarial directions.
    \item We improve over the OoD detection and FP detection performance of state-of-the-art GAN-training based classifiers.
\end{enumerate}
We present in-depth numerical experiments with our method on MNIST, CIFAR10, CIFAR100 and Tiny ImageNet, accompanied with various OoD datasets. We outperform other approaches, also GAN-based ones, in terms of OoD detection and FP detection performance on CIFAR10. Also on MNIST, CIFAR100 and Tiny ImageNet we achieve superior OoD detection performance. Noteworthily, on the more challenging CIFAR10, CIFAR100 and Tiny ImageNet datasets, we achieve significantly stronger model accuracy compared to other approaches based on the same network architecture.

\section{Related Work} \label{sec:relwork}

In this section we give an overview of common uncertainty quantification methods as well as publications related to the different parts of our method. In this context the task of uncertainty quantification is to assign scalar values to predictions, quantifying aleatoric (in-distribution) and epistemic (out-of-distribution) uncertainty.


\textbf{(Baselines)} The works by \cite{XiaCWHS15,HendrycksG17,DeVriesT18} can be considered as early baseline methods for the task of OoD detection. They are frequentist approaches relying on confidence scores gathered from model outputs. The problem can oftentimes be attributed to the usage of the softmax activation function, which leads to overconfident predictions, in particular far away from training data, or to the decoupling of the confidence score from the original classification model during test time. Our proposed method does not make use of the softmax activation function and can produce unified uncertainty estimates during test time without the requirement of auxiliary confidence scores.

\textbf{(Conformal methods)} Methods associated with the term conformal predictions \cite{ParkBML2020,CauchoisGD2021,MessoudiRD2020} are predicting, for a given confidence value, a set of classes that is likely to contain the real class. As such, they can give a more intuitive explanation to which classes might be getting confused. In essence, the uncertainty of a prediction can be derived from the size of the predicted set, assigning higher prediction uncertainty to larger sets. They are however not able to quantify both, aleatoric and epistemic uncertainty, and assigning a scalar uncertainty value to predictions, as in our setting, might not be straight forward.


\textbf{(Data perturbation and auxiliary data)} Many methods use perturbed training examples \cite{LeeLLS18Unified,RenLFSPDDL19,LiangLS18} or auxiliary outlier datasets \cite{HendrycksMD19, KongR21}. \cite{LeeLLS18Unified} use them for their confidence score based on class conditional Gaussian distributions, while \cite{LiangLS18} and \cite{RenLFSPDDL19} are utilizing them to increase the separability between in- and out-of-distribution examples. \cite{HendrycksMD19} use a hand picked auxiliary outlier dataset during model training and \cite{KongR21} use it for selecting a suitable GAN discriminator during training which then serves for OoD detection.
The common problem with using perturbed examples is the sensitivity to hyperparameter selection which might render the resulting examples uninformative.
Additionally, auxiliary outlier datasets cannot always be considered readily available and pose the problem of covering only a small proportion of the real world.
In contrast, our method is able to produce OoC examples that are very close to the in-distribution but still distinguishable from it, thus we do not require explicit data perturbation or any auxiliary outlier datasets.

\textbf{(Bayesian methods)} Bayesian approaches \cite{GalG15a,BlundellCKW15,LakshminarayananPB17} provide a strong theoretical foundation for uncertainty quantification and OoD detection. \cite{LakshminarayananPB17} propose deep ensembles that approximate a distribution over models by averaging predictions of multiple independently trained models. \cite{GalG15a} are utilizing Monte-Carlo (MC) sampling with dropout applied to each layer. 
In \cite{BlundellCKW15} a variational learning algorithm for approximating the intractable posterior distribution over network weights has been proposed. While the theoretical foundation is strong, these methods often require changing the architecture, restricting the model space and/or increased computational cost. While making use of Bayes-rule, we are staying in a frequentist setting and are not dependent on sampling or ensemble techniques. This reduces computational cost and enables our model to produce high quality aleatoric and epistemic uncertainty estimates with a single forward pass. Also, our proposed framework does not change the network architecture, except for the output layer activation function, and thus makes it compatible with previously published techniques.



\textbf{(One-vs-All methods)} One-vs-All methods in the context of OoD detection have been recently studied by \cite{FranchiBADB20,PadhyNRLSL20,SaitoS21}. In the work by \cite{FranchiBADB20} an ensemble of binary neural networks is trained to perform one-vs-all classification on the in-distribution data which are then weighted by a standard softmax classifier. \cite{PadhyNRLSL20} use a DNN with a single sigmoid binary output for every class and explore the possibility of training the one-vs-all network with a distance based loss function instead of the binary cross entropy. Domain adaptation is considered in \cite{SaitoS21}, where they utilize a one-vs-all classifier for a first OoD detection step before classifying into known classes with a second model. Their training objective is also accompanied by a hard negative mining on the in-distribution data. All these methods use the maximum predicted probability as a score for OoD detection and/or classification and do not aggregate the other probabilities into a single score like the method proposed by us. They also do not distinguish into different kinds of uncertainties as in our work. Lastly their training objectives are only based on in-distribution data. Generated OoC data as used in the present work is not considered.


\textbf{(Generative methods)} More recently generative model based methods \cite{SchleglSWSL17,LeeLLS18,SricharanS18,SunZL19,VernekarGDPASC19} have shown strong performance on the task of out-of-distribution detection by supplying classification models with synthesized out-of-distribution examples. \cite{SchleglSWSL17} utilize the latent space of a GAN by gradient based reconstruction of an input example. In the work by \cite{LeeLLS18}, a GAN architecture with an additional classification model is built. The classification model is trained to output a uniform distribution over classes on GAN examples close to the in-distribution. This approach is further improved by \cite{SricharanS18} who show improvements on the task of out-of-distribution detection. The generalization to distant regions in the sample space and the quality of generated boundary examples is however questionable \cite{VernekarGDPASC19}. A similar approach using a normalizing flow and randomly sampled latent vectors is proposed by \cite{GrcicBS21}.
The high level idea of the architectures proposed in the previously mentioned works is similar to the one proposed by us. However, other works are not able to approximate the boundary of data distributions with multiple modes as shown by \cite{VernekarGDPASC19}. Due to the fact that our GAN is class conditional and trained on a low dimensional latent space, we are able to follow multiple distribution modes resulting from different classes. We improve the in-distribution shielding by using a low-dimensional regularizer and have an additional advantage in terms of computational cost as our cGAN model architecture can be chosen with considerably smaller size due to it being trained in the latent space. Furthermore, these methods do not yield separate uncertainty scores for FP and OoD examples.

Generating OoD data based on lower dimensional latent representations has been explored in \cite{VernekarGADSC19a,SensoyKCS20}. \cite{VernekarGADSC19a} utilize a variational Autoencoder (vAE) to produce examples that are inside the encoded manifold (type I) as well as outside of it (type II).
\cite{SensoyKCS20} (GEN) also use a vAE and train a GAN in the resulting latent space, assigning generated examples to the OoD domain in order to estimate a Dirichlet distribution on the class predictions.
Utilizing a vAE has the advantage that one can make assumptions on the distribution of the latent space but also results in slightly blurry reconstructions.
While the work of \cite{SensoyKCS20} is the most similar to our method, we improve on several shortcomings. First of all we employ class conditional models to improve diversity and class shielding.
Additionally, we are able to distinguish aleatoric and epistemic uncertainty while the method by \cite{SensoyKCS20} is not. It assigns the same type of uncertainty to OoD and FP examples.

\section{Method}\label{sec:method}


\subsection{One-vs-All Classification}\label{subsection:onevsall}

We start by formulating our classification model as an ensemble of one-vs-all classifiers. Let $C(o|x,y)$ model the probability that for a given class $y \in \mathcal{Y} = \{1,\ldots,n\}$, an example with features $x\in \mathcal{X} = \mathbb{R}^d$ is OoC. Analogously, $C(i|x,y)=1-C(o|x,y)$ for a given class $y$ models the probability of $x$ being in-class. Let $S\subseteq \mathcal{X}\times\mathcal{Y}$ be our training dataset with $\hat{p}(y)$, $y\in\mathcal{Y}$, the estimated relative class frequencies. To model $C(i|x,y)$, we use a DNN with $n \geq 2$ output neurons equipped with sigmoid activations. For each class output $y$, the data corresponding to class $y$ serves as in-class data and all other data as OoC data. Hence, a basic variant of our training objective is given by a weighted empirical binary cross entropy
\begin{equation}\label{eq:classifier_loss}
    \min_C\,\frac{1}{|S|}\sum_{(x, y)\in S} \left[-\log(C(i|x, y)) -\frac{1}{n-1}\sum_{y' \in \mathcal{Y} \setminus \{y\} }\frac{\hat{p}(y)}{\hat{p}(y')}\log(C(o|x, y'))\right]\,.
\end{equation}
Therein, $\frac{\hat{p}(y)}{\hat{p}(y')}$ is weighting the OoC loss to counter potential class imbalance. Similarly, $\frac{1}{n-1}$ weighs the OoC loss compared to the in-class loss. Applying the transformation
\begin{equation}
    \tilde{C}(i|x,y) = \frac{\frac{1}{n}C(i|x,y)}{\frac{1}{n}C(i|x,y)+\frac{n-1}{n}C(o|x,y)}
\end{equation}
after training, we obtain an appropriate classifier, formalized as follows:
\begin{restatable}[Class posterior]{lemma}{onevsall}\label{lemma:onevsall}
    Under
    \ifdefined\assumptionsref
        typical assumptions of statistical learning theory
    \else
        \cref{assumptions:onevsall}
    \fi
    and training $C$ on \cref{eq:classifier_loss} it holds that for $|S| \to \infty$
    \begin{equation}\label{eq:classifier}
        \hat{p}(y|x)=\frac{\tilde{C}(i|x,y)\hat{p}(y)}{\sum_{y' \in \mathcal{Y}}\tilde{C}(i|x,y')\hat{p}(y')} \longrightarrow p(y|x)\,.
    \end{equation}
\end{restatable}
See \cref{app:theoretic_argument} for the detailed assumptions and the proof, explaining our choice of the classification model.
We distinguish between aleatoric and epistemic uncertainty. While aleatoric uncertainty is considered to be inherent to all observations and irreducible (for a fixed $\mathcal{X}$), epistemic uncertainty is induced by the model choice and the amount of data and can be reduced by additional data and appropriate model selection.
%
%
Using $\hat{p}(y|x)$, we estimate the probability of an example $x$ being in-distribution by defining
\begin{equation} \label{eq:indist_prob}
    \tilde{C}(i|x) = \sum_{y\in\mathcal{Y}} \tilde{C}(i|x,y)\hat{p}(y|x) = \sum_{y\in\mathcal{Y}} \frac{\tilde{C}(i|x,y)^2\hat{p}(y)}{\sum_{y'\in\mathcal{Y}} \tilde{C}(i|x,y')\hat{p}(y')} \, ,
\end{equation}
which yields a quantification of epistemic uncertainty via $\tilde{C}(o|x) = 1 - \tilde{C}(i|x)$. For aleatoric uncertainty estimation, we consider the Shannon entropy of the predicted class probabilities
\begin{equation} \label{eq:entropy}
H(x) = -\sum_{y\in\mathcal{Y}} \hat{p}(y|x) \log( \hat{p}(y|x) ) \, .
\end{equation}
%
%
This is a sensible definition since we showed that $\hat{p}(y|x)\to p(y|x)$ for $|S|\to\infty$. If the real $p(y|x)$ is close to a uniform distribution over the classes (which maximizes the Shannon entropy) we have a high uncertainty and vice versa. For epistemic uncertainty $\tilde{C}(i|x)$ can be considered as a proxy for $p(i|x)$, which results in the epistemic uncertainty by $1-\tilde{C}(i|x)$.
In addition, we generate for each class OoC examples with a conditional GAN. A joint training of classifier and GAN is introduced in the upcoming \cref{subsection:gan-architecture}.

\subsection{GAN Architecture}\label{subsection:gan-architecture}

\begin{figure*}[t]
    \centering
    \includegraphics[width=0.77\textwidth]{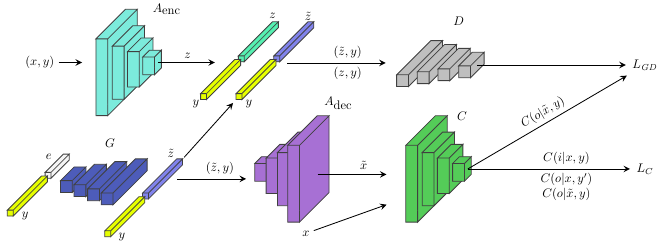}
    \caption{Overview of the proposed architecture. Before training the GAN objective together with the classifier, the cAE is pretrained on the in-distribution training dataset. After that the weights of $A_{\textrm{enc}}$ and $A_{\textrm{dec}}$ are frozen. The flow of information is from left to right, with the top attributed to the GAN loss $L_\mathit{GD}$ and the bottom to the classifier loss $L_C$.}
    \label{fig:overview}
\end{figure*}

Similar to \cite{LeeLLS18}, we combine our classification model with a GAN and train both alternatingly.
The assumptions that we made for \cref{lemma:onevsall} are standard in statistical learning theory, however it is important to notice that we exclude the GAN-generated data from the statement.
By incorporating GAN-generated data into the training of our classification model introduced in \cref{subsection:onevsall}, we are violating this assumption. However, in \cref{sec:experiments} we demonstrate empirically that this does not harm our classification model performance-wise.
The Wasserstein GAN with gradient penalty proposed by \cite{GulrajaniAADC17} serves a basis for our conditional GAN (cGAN). Additionally, we condition the generator as well as the critic on the class labels to generate class specific OoC examples. Inspired by \cite{SensoyKCS20,VernekarGADSC19a} where the latent space of a variational autoencoder (vAE) is utilized for GAN training, we proceed analogously, however using a vanilla conditional autoencoder (cAE) as we observed reconstructions with higher visual quality for this choice. Training the cGAN on a latent space, we do not generate adversarial noise examples and can use less complex generators and critics.
Prior to the cGAN training, we train the cAE on in-distribution data and then freeze the weights during the cGAN and classifier training. The optimization objective of the cAE is given as pixel-wise binary cross-entropy
\begin{equation}
    \min_{A} \frac{1}{|S|}\sum_{(x,y)\in S} -\frac{1}{N_x} \sum_{i=1}^{N_x} x_i\cdot\log(\hat{x}_i) + (1-x_i)\cdot\log(1-\hat{x}_i) \, ,
\end{equation}
with $\hat{x}=A_{\textrm{dec}}(z,y)$ the decoded latent variable, $z=A_{\textrm{enc}}(x,y)$ being the encoded example, $x_i$ the $i$-th pixel of example $x$ and $N_x$ the number of pixels belonging to $x$. Therein, the pixel values are assumed to be in the interval $[0,1]$, while $0\cdot\log(0)=0$.
The cGAN is trained using the objective function
\begin{equation}
    \min_G\max_D\quad \frac{1}{|S|}\sum_{(x, y)\in S} D(z| y) - D(\tilde{z}| y) + \lambda_{gp} \cdot L_{gp} \, ,
\end{equation}
with $D$ the conditional critic, $z=A_{\textrm{enc}}(x,y)$, $\tilde{z}=G(e,y)$ the latent embedding produced by the conditional generator, $e \sim U(0,1)$ noise from a uniform distribution, $y$ a class label and $L_{gp}$ the gradient penalty from \cite{GulrajaniAADC17}. Integrating the classification objective into the cGAN objective, we alternate between
\begin{equation}\label{eq:gan_complete_gd}
    \min_G\max_D\,\overbrace{\frac{1}{|S|}\sum_{(x, y)\in S} D(z| y) - D(\tilde{z}| y) + \lambda_{gp} \cdot L_{gp} -\lambda_{cl}\cdot\log(C(o|\tilde{x},y))+\lambda_R L_R}^{L_\mathit{GD}}
\end{equation}
and
\begin{align}\label{eq:gan_complete_c}
\begin{split}
        \min_C\,&\overbrace{\eqmathboxl{A}{\frac{1}{|S|}\sum_{(x, y)\in S}
        \left[-\log(C(i|x, y))-\frac{\lambda_{\textrm{real}}}{n-1}\left(\sum_{y'\in\mathcal{Y}\setminus\{y\}}\frac{\hat{p}(y)}{\hat{p}(y')}\log(C(o|x, y'))\right)\right.}}^{L_C}\\
        &\eqmathboxl{A}{\hphantom{\frac{1}{|S|}\sum_{(x, y)\in S}}\left.\vphantom{\frac{\lambda_{\textrm{real}}}{n-1}\left(\sum_{y'\in\mathcal{Y}\setminus\{y\}}\frac{\hat{p}(y)}{\hat{p}(y')}\log(C(o|x, y'))\right)}-(1-\lambda_{\textrm{real}})\cdot\log(C(o|\tilde{x}, y))\right] \,,}
\end{split}
\end{align}
with $\lambda_{\textrm{real}}\in[0,1]$ being an interpolation factor between real and generated OoC examples and $L_R$ being an additional regularization loss for the generated latent codes with hyperparameter $\lambda_R \geq 0$, which we introduce in \cref{subsection:regloss}.
The latent embeddings produced by the cGAN are decoded with the pretrained cAE, thus $\tilde{x} = A_{\textrm{dec}}(\tilde{z},y) = A_{\textrm{dec}}(G(e,y),y)$. That is, the cGAN is trained on the latent space while the classification model is trained on the original feature space. Our entire GAN-architecture is visualized in \cref{fig:overview}.

\subsection{Low-Dimensional Regularizer}\label{subsection:regloss}

In low dimensional latent spaces we found it to be advantageous to apply an additional regularizer to the generated latent embeddings $\tilde{z}$ to improve class shielding. Let $(z,\,y)$ be the latent embedding of an example $x$ with its corresponding class label $y$ and $\mathcal{Z}(z,y)=\{\tilde{z}-z|\,\tilde{z}=G(e,y),\,e\sim U(0,1)\}=\left\{\bar{z}^1,\ldots,\bar{z}^{N_y^z}\right\}$ all generated latent codes with the same class label and normalized to origin $z$. We encourage the generator to produce latent codes that more uniformly shield the class $y$ by maximizing the average angular distance between all $\bar{z}\in\mathcal{Z}(z,y)$, which corresponds to minimizing
\begin{equation}\label{eq:singleregloss}
    l_R(z,y)=\frac{2}{N_y^z\cdot(N_y^z-1)}\cdot \sum_{\substack{\bar{z}^i, \bar{z}^j \in \mathcal{Z}(z,y)\\ i< j}} -\log\left(\arccos\left(\frac{\bar{z}^i\ast\bar{z}^j}{\lVert\bar{z}^i\rVert\cdot\lVert\bar{z}^j\rVert}\right)\cdot\frac{1}{\pi}\right) \, ,
\end{equation}
with $\ast$ being the dot-product.
The logarithm introduces an exponential scaling for very small angular distances, encouraging a more evenly spread distribution of the generated latent codes. This loss is then averaged over all class labels and training data examples
\begin{equation} \label{eq:regloss}
    L_R = \frac{1}{n} \sum_{y\in\mathcal{Y}} \left[\frac{1}{N_y} \sum_{(x,y)\in S} l_R(A_{\textrm{enc}}(x,y),y)\right] \, ,
\end{equation}
with $N_y=|\{(x,\,y') | y'=y\}|$ the number of examples with class label $y$. In \cref{sec:regularizer_explained} a more detailed explanation of this regularization loss is provided.

During our experiments we studied different regularizer losses such as Manhattan/Euclidean distance, infinity norm and standard cosine similarity. In experiments, we found that for our purpose \cref{eq:regloss} performed best. We argue that this can be attributed to the independence of the latent space value range, which can have a large impact on the $p$-norm distance metrics, and to the exponential scaling for very small angular distances.

To this end, the basic variant of our method does not account for model uncertainty. Thus, we also include results with MC-Dropout, which also demonstrates the compatibility of our model with existing methods. Another limitation of our approach is that we are not explicitly accounting for adversarial attacks because during training in the cAE latent space we are removing any adversarial directions.

\section{Experiments} \label{sec:experiments}

We compare our method with a number of related approaches (cf.~\cref{tab:rel_methods}).
While the first two methods are simple baselines, the subsequent three ones are Bayesian ones and the final two methods are GAN-based (cf.~\cref{sec:relwork}).
\begin{figure}
    \centering
    \begin{minipage}{0.37\textwidth}
        \centering
            \begin{tabular}{lc}
                \toprule
                Method & Reference\\
                \midrule
                Maximum softmax & \cite{HendrycksG17} \\
                Entropy & \\
                Bayes-by-Backprop & \cite{BlundellCKW15}\\
                MC-Dropout & \cite{GalG15a}\\
                Deep-Ensembles & \cite{LakshminarayananPB17}\\
                Confident Classifier & \cite{LeeLLS18}\\
                GEN & \cite{SensoyKCS20}\\
                \bottomrule
            \end{tabular}
        \captionof{table}{Baselines and related methods used for comparing our work.}
        \label{tab:rel_methods}
    \end{minipage}\hspace{0.04\textwidth}%
    \begin{minipage}{0.59\textwidth}
        \centering
        \begin{minipage}{0.45\textwidth}
            \centering
            \includegraphics[width=\textwidth,keepaspectratio]{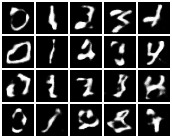}
            a)
        \end{minipage}\hspace{0.02\textwidth}%
        \begin{minipage}{0.45\textwidth}
            \centering
            \includegraphics[width=\textwidth,keepaspectratio]{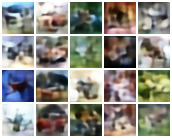}
            b)
        \end{minipage}
        \captionof{figure}{Generated OoC examples by our approach. a) MNIST examples for digit classes 0-4 (left to right). b) CIFAR10 examples for classes \textit{airplane}, \textit{automobile}, \textit{bird}, \textit{cat} and \textit{deer} (left to right). See \cref{fig:samples_different_lambda_reg} for more examples.}
        \label{fig:ood_samples}
    \end{minipage}
\end{figure}
Following the publications \cite{SensoyKCS20,VernekarGADSC19a,RenLFSPDDL19,HendrycksG17}, we consider four experimental setups using the MNIST ($28\times28$, $10$ classes) \cite{LeCunBBH98}, CIFAR10 ($32\times32$, $10$ classes) \cite{krizhevsky2009}, CIFAR100 ($32\times32$, $100$ classes) \cite{krizhevsky2009} and Tiny ImageNet ($64\times64$, $200$ classes) \cite{LeY2015} datasets as in-distribution, respectively. Similar to \cite{SensoyKCS20,VernekarGADSC19a} and others, we split the datasets class-wise into two non-overlapping sets, i.e., MNIST 0-4 / 5-9, CIFAR10 0-4 / 5-9, CIFAR100 0-49 / 50-99 and Tiny ImageNet 0-99 / 100-199. While the first half serves as in-distribution data, the second half constitutes  OoD cases close to the training data and therefore difficult to detect. For the MNIST 0-4 dataset, we consider the MNIST 5-9, EMNIST-Letters \cite{CohenATS17}, Fashion-MNIST \cite{XiaoRV2017}, Omniglot \cite{LakeST2015}, SVHN \cite{NetzerWCB2011} and CIFAR10 datasets as OoD examples. For the CIFAR10 0-4 datasets, we use CIFAR10 5-9, LSUN \cite{YuZSS2015}, SVHN, Fashion-MNIST and MNIST as OoD examples. The same OoD datasets are used for CIFAR100 except CIFAR10 5-9 being replaced by CIFAR100 50-99. For the Tiny ImageNet 0-99 dataset we use Tiny ImageNet 100-199, SVHN, Fashion-MNIST and MNIST as OoD examples. These selections yield compositions of training and OoD examples with strongly varying difficulty for state-of-the-art OoD detection. Besides that, we examine our method's behavior on a 2D toy example with two overlapping Gaussians (having trivial covariance structure), see~\cref{fig:toy_example}.
Additionally, we split the official training sets into $80\%$ / $20\%$ training / validation sets, where the latter are used for hyperparameter tuning and model selection.
%
Like related works, we utilize the LeNet-5 architecture on MNIST and a small ResNet on CIFAR10, CIFAR100 and Tiny ImageNet as classification models. To ensure fair conditions, we re-implemented all aforementioned methods while following the authors recommendations for hyperparameters and their reference implementations. For methods involving more complex architectures, e.g.\ a GAN or a VAE as in \cite{LeeLLS18,SensoyKCS20}, we used the proposed architectures for those components, while for the sake of comparability sticking to our choice of classifier models. All implementations are using PyTorch \cite{PaszkeGMLB19} and can be found in the associated repository\footnote{\url{https://github.com/RonMcKay/UQGAN}}. For each method, we selected the network checkpoint with maximal validation accuracy during training. For a more detailed overview of the hyperparameters used in our experiments, we refer to \cref{app:hyperparams} or the implementation.

For evaluation we use the following well established metrics:
\begin{itemize}
    \setlength\itemsep{-0.2em}
    \item Classification accuracy on the in-distribution datasets.
    \item Area under the Receiver Operating Characteristic Curve (AUROC). We apply the AUROC to the binary classification tasks in-/out-of-distribution (via the score from \cref{eq:indist_prob}) and TP/FP (Success/Failure) (via the score from \cref{eq:entropy}).
    \item Expected Calibration Error (ECE) \cite{NaeiniCH15} applied to the estimated class probabilities $\hat{p}(y|x)$ for in-distribution examples $x$, computed on $15$ bins.
    \item Area under the Precision Recall Curve (AUPR) w.r.t.\ the binary in-/out-of-distribution decision (via the score from \cref{eq:indist_prob}).
    We further distinguish between AUPR-In and AUPR-Out. For AUPR-in the in-distribution class is the positive one, while for AUPR-Out the out-of-distribution class is the positive one.
    \item FPR @ 95\% TPR computes the False Positive Rate (FPR) at the decision threshold on the OoD score from \cref{eq:indist_prob} ensuring a True Positive Rate (TPR) of 95\%.
\end{itemize}
%
%

Before discussing our results on MNIST, CIFAR10, CIFAR100 and Tiny ImageNet, we briefly discuss our findings on the 2D example.
As can be seen in \cref{fig:toy_example}, the generated OoC examples are nicely shielding the respective in-distribution classes. OoC examples of one class can be in-distribution examples of other classes. This is an intended feature and to this end, the loss term for the synthesized OoC examples in \cref{eq:gan_complete_c} is class conditional. This feature is supposed to make our one-vs-all classifier predict a high OoC probability in the OoC regime. One can also observe that the estimated epistemic and aleatoric uncertainties are complementary, resulting in a high aleatoric uncertainty in the overlapping region of the Gaussians while also having a low epistemic uncertainty there. This is one of the main advantages that sets our approach apart from related methods.
Results for a slightly more challenging 2D toy example on the two moons dataset and on a grid of Gaussians are presented in \cref{app:additional_toy}.
We now demonstrate that this result generalizes to higher dimensional problems.

\begin{table}[t]
    \setlength{\tabcolsep}{1pt}
    \caption{Results for MNIST (0-4) as in-distribution vs \{MNIST (5-9), EMNIST-Letters, Omniglot, Fashion-MNIST, SVHN, CIFAR10\} as out-of-distribution datasets}
    \label{tab:res_mnist}
    \centering
    \scalebox{\resultstablescaling}{
    \begin{tabular}{l @{\hskip 16pt} ccc @{\hskip 16pt} cccc}
        \toprule
        \multirow{2}{*}[-5pt]{Method} & \multicolumn{3}{@{}c@{\hskip 16pt}}{In-Distribution} & \multicolumn{4}{c}{Out-of-Distribution}\vspace{5pt}\\
         & Accuracy $\uparrow$ & AUROC S/F $\uparrow$ & ECE $\downarrow$ & AUROC $\uparrow$ & AUPR-In $\uparrow$ & AUPR-Out $\uparrow$ & FPR@95\% TPR $\downarrow$\\
         \midrule
         Ours & $99.74\,(0.05)$ & $99.35\,(0.31)$ & $0.15\,(0.05)$ & $98.03\,(0.28)$ & $80.05\,(2.65)$ & $99.87\,(0.02)$ & $8.73\,(1.47)$\\
         Ours with MC-Dropout & $99.80\,(0.04)$ & $99.42\,(0.11)$ & $1.38\,(0.11)$ & $\bm{98.58\,(0.25)}$ & $\bm{83.71\,(2.40)}$ & $\bm{99.91\,(0.02)}$ & $\bm{5.60\,(0.77)}$\\
         \midrule
         One-vs-All Baseline & $99.84\,(0.06)$ & $\bm{99.84\,(0.06)}$ & $0.12\,(0.04)$ & $97.12\,(0.17)$ & $66.68\,(1.93)$ & $99.81\,(0.01)$ & $9.45\,(0.56)$\\
         Max. Softmax \cite{HendrycksG17} & $99.87\,(0.02)$ & $99.68\,(0.13)$ & $\bm{0.11\,(0.02)}$ & $97.07\,(0.12)$ & $69.00\,(1.65)$ & $99.81\,(0.01)$ & $9.71\,(0.37)$\\
         Entropy & $99.87\,(0.02)$ & $99.66\,(0.14)$ & $\bm{0.11\,(0.02)}$ & $97.13\,(0.12)$ & $68.76\,(1.94)$ & $99.81\,(0.01)$ & $9.65\,(0.39)$\\
         Bayes-by-Backprop \cite{BlundellCKW15} & $99.67\,(0.02)$ & $99.50\,(0.06)$ & $0.78\,(0.05)$ & $95.46\,(0.26)$ & $67.09\,(2.70)$ & $99.60\,(0.03)$ & $17.33\,(1.06)$\\
         MC-Dropout \cite{GalG15a} & $99.91\,(0.02)$ & $99.62\,(0.12)$ & $0.83\,(0.10)$ & $97.69\,(0.16)$ & $72.82\,(2.55)$ & $99.86\,(0.01)$ & $8.28\,(0.39)$\\
         Deep-Ensembles \cite{LakshminarayananPB17} & $\bm{99.89\,(0.03)}$ & $99.74\,(0.08)$ & $0.15\,(0.01)$ & $97.70\,(0.03)$ & $73.09\,(0.62)$ & $99.86\,(0.00)$ & $7.81\,(0.21)$\\
         Confident Classifier \cite{LeeLLS18} & $99.82\,(0.02)$ & $99.62\,(0.15)$ & $0.16\,(0.01)$ & $98.15\,(0.13)$ & $78.31\,(2.01)$ & $99.88\,(0.01)$ & $7.65\,(0.46)$\\
         GEN \cite{SensoyKCS20} & $99.70\,(0.03)$ & $98.13\,(0.45)$ & $1.98\,(0.85)$ & $97.78\,(0.70)$ & $69.77\,(10.20)$ & $99.86\,(0.04)$ & $8.98\,(1.90)$\\
         \midrule
         Entropy Oracle & $99.79\,(0.06)$ & $98.93\,(0.68)$ & $0.82\,(0.06)$ & $99.90\,(0.02)$ & $98.66\,(0.28)$ & $99.99\,(0.00)$ & $0.43\,(0.10)$\\
         One-vs-All Oracle & $99.77\,(0.03)$ & $99.47\,(0.16)$ & $0.14\,(0.02)$ & $99.90\,(0.01)$ & $98.50\,(0.12)$ & $99.99\,(0.00)$ & $0.40\,(0.03)$\\
         \bottomrule
    \end{tabular}
    }
    \vspace{-1em}
\end{table}
\begin{table}[t]
    \setlength{\tabcolsep}{1pt}
    \caption{Results for CIFAR10 (0-4) as in-distribution vs \{CIFAR10 (5-9), LSUN, SVHN, Fashion-MNIST, MNIST\} as out-of-distribution datasets}
    \label{tab:res_cifar10}
    \centering
    \scalebox{\resultstablescaling}{
    \begin{tabular}{l @{\hskip 16pt} ccc @{\hskip 16pt} cccc}
        \toprule
        \multirow{2}{*}[-5pt]{Method} & \multicolumn{3}{@{}c@{\hskip 16pt}}{In-Distribution} & \multicolumn{4}{c}{Out-of-Distribution}\vspace{5pt}\\
         & Accuracy $\uparrow$ & AUROC S/F $\uparrow$ & ECE $\downarrow$ & AUROC $\uparrow$ & AUPR-In $\uparrow$ & AUPR-Out $\uparrow$ & FPR@95\% TPR $\downarrow$\\
         \midrule
         Ours & $87.26\,(0.29)$ & $84.71\,(0.52)$ & $10.35\,(0.20)$ & $86.49\,(0.63)$ & $49.08\,(1.05)$ & $98.72\,(0.08)$ & $45.78\,(2.90)$\\
         Ours with MC-Dropout & $\bm{90.26\,(0.22)}$ & $\bm{89.19\,(0.13)}$ & $\bm{2.34\,(0.33)}$ & $\bm{89.64\,(0.23)}$ & $\bm{53.15\,(0.27)}$ & $\bm{99.01\,(0.03)}$ & $\bm{43.54\,(1.56)}$\\
         \midrule
         One-vs-All Baseline & $82.82\,(0.62)$ & $82.19\,(0.81)$ & $8.62\,(4.55)$ & $72.52\,(2.16)$ & $32.24\,(2.14)$ & $96.01\,(0.42)$ & $88.74\,(1.86)$\\
         Max. Softmax \cite{HendrycksG17} & $82.42\,(0.31)$ & $83.29\,(0.89)$ & $11.34\,(0.83)$ & $72.52\,(0.51)$ & $30.52\,(0.97)$ & $96.10\,(0.05)$ & $87.68\,(0.43)$\\
         Entropy & $82.42\,(0.31)$ & $83.41\,(0.88)$ & $11.34\,(0.83)$ & $72.85\,(0.49)$ & $30.43\,(0.87)$ & $96.21\,(0.06)$ & $85.41\,(0.89)$\\
         Bayes-by-Backprop \cite{BlundellCKW15} & $84.05\,(0.33)$ & $85.22\,(0.40)$ & $9.17\,(0.41)$ & $74.23\,(0.96)$ & $29.91\,(1.61)$ & $96.48\,(0.19)$ & $83.97\,(1.47)$\\
         MC-Dropout \cite{GalG15a} & $85.08\,(0.56)$ & $83.91\,(0.49)$ & $9.90\,(0.42)$ & $77.56\,(1.27)$ & $38.75\,(1.40)$ & $96.85\,(0.17)$ & $82.35\,(1.08)$\\
         Deep-Ensembles \cite{LakshminarayananPB17} & $85.43\,(0.22)$ & $85.29\,(0.57)$ & $3.10\,(0.29)$ & $74.24\,(0.73)$ & $32.81\,(1.50)$ & $96.43\,(0.10)$ & $85.07\,(0.81)$\\
         Confident Classifier \cite{LeeLLS18} & $83.58\,(0.11)$ & $85.08\,(0.18)$ & $9.31\,(0.80)$ & $73.33\,(0.53)$ & $32.32\,(1.09)$ & $96.29\,(0.11)$ & $85.04\,(1.09)$\\
         GEN \cite{SensoyKCS20} & $82.46\,(0.35)$ & $82.88\,(0.49)$ & $6.71\,(1.81)$ & $86.01\,(1.60)$ & $42.32\,(2.81)$ & $98.66\,(0.17)$ & $45.39\,(3.26)$\\
         \midrule
         Entropy Oracle & $83.41\,(0.58)$ & $80.73\,(0.54)$ & $7.53\,(0.27)$ & $95.44\,(0.28)$ & $68.51\,(0.57)$ & $99.57\,(0.04)$ & $17.27\,(1.44)$\\
         One-vs-All Oracle & $83.70\,(0.50)$ & $81.27\,(1.03)$ & $8.85\,(0.49)$ & $91.38\,(0.74)$ & $53.75\,(1.49)$ & $99.15\,(0.08)$ & $35.94\,(3.79)$\\
         \bottomrule
    \end{tabular}
    }
\end{table}
For MNIST and CIFAR10, \cref{fig:ood_samples} shows the OoC examples produced at the end of the generator training. Due to using a conditional GAN and an AE, we are able to generate OoC examples (instead of only out-of-distribution as in related works) during test time. It can be seen that the resulting examples resemble a lot of semantic similarities with the original class while still being distinguishable from them.


All presented results are computed on the respective (official) test sets of the datasets. We also conducted an extensive parameter study on the validation sets, which is summarized in \cref{app:parameter_study}. The conclusion of this parameter study is that the performance of our framework is in general stable w.r.t.\ the choice of the hyperparameters. Increasing $\lambda_{\textrm{reg}}$ positively impacts the model performance up to a certain maximum. The best performance is obtained by choosing latent dimensions such that the cAE is able to compute reconstructions of good visual quality. Across all datasets, choosing $\lambda_{\textrm{real}} \in [0.5,0.6]$ achieves the best detection scores, also indicating a positive influence of the generated OoC examples on the model's classification accuracy.

\Cref{tab:res_mnist,tab:res_cifar10,tab:res_cifar100,tab:res_tinyimagenet} present the results of our experiments. The scores in the columns of the section \textit{In-Distribution} are solely computed on the respective in-distribution dataset.
The scores in the columns of the section \textit{Out-of-Distribution} are displaying the OoD detection performance when presenting examples from the respective in-distribution dataset as well as from the entirety of all assigned OoD datasets.
Note that we did not apply any balancing in the OoD datasets but included the respective test sets as is (see \cref{app:hyperparams} for the sizes of the test sets). As an upper bound on the OoD detection performance we also show results for two oracle models, supplied with the real OoD training datasets they are evaluated on. One of them is trained with the standard softmax and binary-cross-entropy to maximize entropy on OoD examples and the other one with our proposed loss function, cf.\ \cref{eq:gan_complete_c}.

\begin{table}[t]
    \setlength{\tabcolsep}{1pt}
    \caption{Results for CIFAR100 (0-49) as in-distribution vs \{CIFAR100 (50-99), LSUN, SVHN, Fashion-MNIST, MNIST\} as out-of-distribution datasets}
    \label{tab:res_cifar100}
    \centering
    \scalebox{\resultstablescaling}{
    \begin{tabular}{l @{\hskip 16pt} ccc @{\hskip 16pt} cccc}
        \toprule
        \multirow{2}{*}[-5pt]{Method} & \multicolumn{3}{@{}c@{\hskip 16pt}}{In-Distribution} & \multicolumn{4}{c}{Out-of-Distribution}\vspace{5pt}\\
         & Accuracy $\uparrow$ & AUROC S/F $\uparrow$ & ECE $\downarrow$ & AUROC $\uparrow$ & AUPR-In $\uparrow$ & AUPR-Out $\uparrow$ & FPR@95\% TPR $\downarrow$\\
         \midrule
         Ours & $56.60\,(0.73)$ & $73.61\,(0.33)$ & $34.32\,(0.42)$ & $80.11\,(1.40)$ & $28.81\,(1.32)$ & $97.98\,(0.19)$ & $\bm{55.23\,(3.20)}$\\
         Ours with MC-Dropout & $\bm{64.53\,(0.47)}$ & $80.38\,(0.40)$ & $10.92\,(0.24)$ & $\bm{80.75\,(1.19)}$ & $\bm{31.75\,(1.53)}$ & $\bm{98.04\,(0.14)}$ & $58.10\,(2.11)$\\
         \midrule
         One-vs-All Baseline & $50.90\,(0.83)$ & $76.21\,(0.91)$ & $23.89\,(2.16)$ & $62.99\,(1.59)$ & $15.69\,(2.69)$ & $94.53\,(0.25)$ & $92.44\,(0.76)$\\
         Max. Softmax \cite{HendrycksG17} & $56.12\,(0.60)$ & $80.68\,(0.29)$ & $19.10\,(3.23)$ & $67.68\,(1.56)$ & $23.03\,(1.71)$ & $95.47\,(0.30)$ & $88.42\,(1.42)$\\
         Entropy & $56.12\,(0.60)$ & $81.16\,(0.26)$ & $19.10\,(3.23)$ & $69.43\,(1.71)$ & $23.75\,(1.85)$ & $95.77\,(0.32)$ & $87.08\,(1.74)$\\
         Bayes-by-Backprop \cite{BlundellCKW15} & $56.02\,(0.50)$ & $81.90\,(0.55)$ & $14.71\,(0.31)$ & $69.74\,(0.76)$ & $24.60\,(1.25)$ & $95.86\,(0.12)$ & $87.01\,(0.75)$\\
         MC-Dropout \cite{GalG15a} & $59.88\,(0.81)$ & $82.57\,(0.60)$ & $21.94\,(0.58)$ & $67.75\,(1.15)$ & $22.31\,(1.14)$ & $95.40\,(0.22)$ & $89.38\,(1.30)$\\
         Deep-Ensembles \cite{LakshminarayananPB17} & $62.36\,(0.43)$ & $\bm{82.77\,(0.47)}$ & $\bm{2.72\,(0.37)}$ & $74.29\,(0.50)$ & $29.38\,(0.67)$ & $96.53\,(0.13)$ & $83.37\,(1.41)$\\
         Confident Classifier \cite{LeeLLS18} & $54.16\,(0.13)$ & $81.07\,(0.44)$ & $27.19\,(5.45)$ & $68.66\,(0.48)$ & $22.51\,(0.20)$ & $95.75\,(0.08)$ & $86.17\,(0.55)$\\
         GEN \cite{SensoyKCS20} & $51.16\,(0.26)$ & $77.83\,(0.80)$ & $41.26\,(2.05)$ & $77.43\,(2.58)$ & $26.12\,(2.41)$ & $97.64\,(0.33)$ & $61.48\,(4.33)$\\
         \midrule
         Entropy Oracle & $53.66\,(0.39)$ & $81.39\,(0.69)$ & $19.82\,(0.34)$ & $86.16\,(0.41)$ & $38.63\,(0.58)$ & $98.68\,(0.05)$ & $44.67\,(1.46)$\\
         One-vs-All Oracle & $51.82\,(1.81)$ & $76.43\,(0.83)$ & $15.26\,(4.27)$ & $92.91\,(0.28)$ & $47.68\,(1.37)$ & $99.39\,(0.03)$ & $22.35\,(0.71)$\\
         \bottomrule
    \end{tabular}
    }
    \vspace{-1em}
\end{table}
\begin{table}[t]
    \setlength{\tabcolsep}{1pt}
    \caption{Results for Tiny ImageNet (0-99) as in-distribution vs \{Tiny ImageNet (100-199), SVHN, Fashion-MNIST, MNIST\} as out-of-distribution datasets.}
    \label{tab:res_tinyimagenet}
    \centering
    \scalebox{\resultstablescaling}{
    \begin{tabular}{l @{\hskip 16pt} ccc @{\hskip 16pt} cccc}
        \toprule
        \multirow{2}{*}[-5pt]{Method} & \multicolumn{3}{@{}c@{\hskip 16pt}}{In-Distribution} & \multicolumn{4}{c}{Out-of-Distribution}\vspace{5pt}\\
         & Accuracy $\uparrow$ & AUROC S/F $\uparrow$ & ECE $\downarrow$ & AUROC $\uparrow$ & AUPR-In $\uparrow$ & AUPR-Out $\uparrow$ & FPR@95\% TPR $\downarrow$\\
         \midrule
         Ours & $34.28\,(0.37)$ & $71.90\,(0.47)$ & $48.94\,(0.66)$ & $79.25\,(1.61)$ & $26.35\,(2.25)$ & $97.66\,(0.20)$ & $47.22\,(1.85)$\\
         Ours with MC-Dropout & $\bm{45.60\,(0.43)}$ & $79.18\,(0.42)$ & $5.92\,(0.38)$ & $\bm{94.96\,(0.13)}$ & $\bm{59.76\,(0.64)}$ & $\bm{99.51\,(0.01)}$ & $\bm{13.72\,(0.30)}$\\
         \midrule
         One-vs-All Baseline & $35.18\,(0.26)$ & $76.23\,(0.43)$ & $11.59\,(4.41)$ & $55.19\,(2.29)$ & $17.97\,(2.41)$ & $90.97\,(0.69)$ & $97.32\,(0.62)$\\
         Max. Softmax \cite{HendrycksG17} & $36.06\,(0.30)$ & $78.56\,(0.68)$ & $26.01\,(7.25)$ & $61.53\,(1.04)$ & $21.07\,(0.98)$ & $93.29\,(0.29)$ & $92.51\,(0.87)$\\
         Entropy & $36.06\,(0.30)$ & $79.39\,(0.74)$ & $26.01\,(7.25)$ & $62.44\,(1.31)$ & $21.90\,(0.86)$ & $93.16\,(0.46)$ & $93.79\,(1.28)$\\
         Bayes-by-Backprop \cite{BlundellCKW15} & $32.31\,(0.43)$ & $78.44\,(0.91)$ & $19.39\,(0.62)$ & $68.05\,(2.29)$ & $21.24\,(2.14)$ & $95.23\,(0.52)$ & $81.70\,(3.53)$\\
         MC-Dropout \cite{GalG15a} & $43.48\,(0.53)$ & $80.63\,(0.30)$ & $\bm{2.79\,(0.35)}$ & $63.35\,(4.23)$ & $27.11\,(2.29)$ & $92.78\,(1.06)$ & $95.86\,(1.28)$\\
         Deep-Ensembles \cite{LakshminarayananPB17} & $42.48\,(0.22)$ & $\bm{81.29\,(0.38)}$ & $16.50\,(6.94)$ & $67.76\,(0.27)$ & $30.85\,(0.42)$ & $93.93\,(0.07)$ & $93.79\,(0.23)$\\
         Confident Classifier \cite{LeeLLS18} & $36.07\,(0.53)$ & $78.66\,(0.67)$ & $35.48\,(2.13)$ & $59.99\,(1.84)$ & $20.58\,(1.61)$ & $92.66\,(0.39)$ & $94.49\,(0.51)$\\
         GEN \cite{SensoyKCS20} & $30.54\,(0.84)$ & $73.40\,(0.91)$ & $28.79\,(0.79)$ & $84.65\,(5.42)$ & $36.86\,(8.25)$ & $98.28\,(0.69)$ & $39.67\,(12.50)$\\
         \midrule
         Entropy Oracle & $37.18\,(0.50)$ & $79.50\,(0.63)$ & $17.05\,(2.82)$ & $85.43\,(1.90)$ & $41.95\,(2.31)$ & $98.27\,(0.29)$ & $48.74\,(6.63)$\\
         One-vs-All Oracle & $34.92\,(0.56)$ & $75.10\,(0.92)$ & $20.61\,(3.50)$ & $95.75\,(0.09)$ & $59.69\,(0.50)$ & $99.59\,(0.01)$ & $10.30\,(0.46)$\\
         \bottomrule
    \end{tabular}
    }
    \vspace{-1em}
\end{table}
We first discuss the in-distribution performance of our method. W.r.t.\ MNIST, the results given in the left section of \cref{tab:res_mnist} show that we are on par with state-of-the-art GAN-based approaches while still having a similar ECE, only being surpassed by Deep-Ensembles and the other baselines by a fairly small margin. However, considering the respective CIFAR10 results in the left section of \cref{tab:res_cifar10}, we clearly outperform state-of-the-art GAN-based methods as well as all other baseline methods by a large margin. Noteworthily, we achieve an accuracy of $90.26\%$ which is $5$ to $8$ percent points (pp.) above the other classifiers and an AUROC S/F of $89.19\%$ which is $4$ to $6$ pp.\ higher than for the other methods. This corresponds to a relative improvement of $33\%$ in accuracy and $26\%$ in AUROC S/F compared to the second best method. A similar improvement in accuracy can be observed for CIFAR100 and Tiny ImageNet (\cref{tab:res_cifar100,tab:res_tinyimagenet}) where the AUROC S/F is slightly lower but still on par with other approaches. Considering ECE on CIFAR10/100 and Tiny ImageNet, our MC-Dropout variant is outperforming nearly all related methods, except for CIFAR100 and Tiny ImageNet where it is ranked second. This signals that, although we incorporate generated OoC examples impurifying the distribution of training data presented to the classifier, empirically there is no evidence that this harms the learned classifier, neither w.r.t.\ calibration, nor w.r.t.\ separation.

Considering the OoD results from the right-hand sections of \cref{tab:res_mnist,tab:res_cifar10,tab:res_cifar100,tab:res_tinyimagenet}, the superiority of our method compared to the other ones is now consistent over all four in-distribution datasets. On the MNIST dataset we are outperforming previously published works, especially considering the AUPR-In and FPR @ 95\% TPR metrics, with a $25\%$ and $27\%$ relative improvement over the second best method, respectively. This is consistent with the results for CIFAR10 and CIFAR100 as in-distribution datasets where we achieve for both AUROC and AUPR-In a relative improvement of $19\%$--$26\%$ and $3\%$--$15\%$, respectively, over the second best method. On the higher dimensional Tiny ImageNet dataset, which has $20\times$ more classes than MNIST/CIFAR10 and $4\times$ more input pixels, the results are even stronger with an $36\%$--$67\%$ relative improvement over the second best competitor.

Comparing our results with the ones of the oracles, two observations become apparent. Firstly, in some OoD experiments the GAN-generated OoC examples achieve results fairly close to or even surpassing the ones of the oracles while also in some of them there is still room for improvement left (in particular w.r.t.\ FPR@95\%{}TPR).
Secondly, GAN-generated OoC examples can help improve generalization (in terms of classification accuracy) while real OoD data might be too far away from the in-distribution data.
%
%
An OoD-dataset-wise breakdown of the results in \cref{tab:res_mnist,tab:res_cifar10,tab:res_cifar100,tab:res_tinyimagenet} is provided in \cref{app:ood_datasets}. For MNIST, this breakdown reveals that our method performs particularly well in the difficult task of separating MNIST 0-4 and MNIST 5-9. On the other MNIST-related tasks we achieve mid-tier results, being slightly behind the other GAN-based methods. The same holds for CIFAR100 0-49 and CIFAR100 50-99.
However, with regards to CIFAR10 and Tiny ImageNet we are consistently outperforming the other methods by large margins.

In \cref{app:boosting} we present results of an experiment where we perform FP and OoD detection jointly on the computed uncertainty scores. The main observations there are that our method outperforms other GAN-based methods and that our method including dropout achieves the overall best detection performance.

\section{Conclusion}\label{sec:conclusion}

In this work, we introduced a GAN-based model yielding a one-vs-all classifier with complete uncertainty quantification. Our model distinguishes uncertainty between classes (in large sample limit approaching aleatoric uncertainty) from OoD uncertainty (approaching epistemic uncertainty). We have demonstrated in numerical experiments that our model sets a new state-of-the-art w.r.t.\ OoD as well as FP detection. The generated OoC examples do not harm the training success in terms of calibration, but even improve it in terms of accuracy. We have seen that incorporating MC dropout to account for model uncertainty can further improve the results.

\subsection*{Broader Impact}

We are contributing generally applicable methodology which can be used in any real world application. The general direction of making the deployment of models more feasible due to improved quality of uncertainties might ultimately result in a reduction of the number of jobs in the respective sectors. However, as of now, it is also important to note that this method cannot give any guarantees on the estimated uncertainties and thus these should be used with caution and never solely relied upon.

\subsection*{Acknowledgment}
M.R.\ acknowledges useful and interesting discussions with Hanno Gottschalk.

\bibliographystyle{plain}
\bibliography{references}


\newpage
\appendix

\section{Theoretical Consideration of the One-vs-All Classifier} \label{app:theoretic_argument}

Let $C(i|x,y),\,y\in\{1,\ldots,n\}$ be an ensemble of binary classifiers with $C(o|x,y)=1-C(i|x,y)$. Furthermore we define the joint distribution
\begin{equation}
    p_U(x,y)=p(x|y)p_U(y)
\end{equation}
w.r.t.\ a uniform class distribution, i.e., $p_U(y) \equiv \frac{1}{n} \; \forall y\in\mathcal{Y}$. Lastly we define
\begin{equation}
    \tilde{C}(i|x,y)=\frac{\frac{1}{n}C(i|x,y)}{\frac{1}{n}C(i|x,y)+\frac{n-1}{n}C(o|x,y)}\,.
\end{equation}
Additionally, $\hat{p}(y),\,\forall y\in\mathcal{Y}$, are the estimated relative class frequencies. Recalling \eqref{eq:classifier_loss}, our proposed training objective for the ensemble of binary classifiers is defined by
\begin{equation*}
    \min_C\,\frac{1}{|S|}\sum_{(x, y)\in S} \left[-\log(C(i|x, y)) -\frac{1}{n-1}\sum_{y' \in \mathcal{Y} \setminus \{ y \}}\frac{\hat{p}(y)}{\hat{p}(y')}\log(C(o|x, y'))\right]\, .\tag{\ref*{eq:classifier_loss}}
\end{equation*}

\begin{assumption}\label{assumptions:onevsall}
We make the following assumptions for our one-vs-all classifier $C$
\begin{enumerate}
    \item We sample $S \sim p(x,y)$ i.i.d.\ and there is no GAN-generated data involved;
    \item Let $\mathcal{H} = \{ C \}$ the set of all realizable one-vs-all classifiers. We assume, there exists a $C^* \in \mathcal{H}$ such that $\hat{p}(y|x)=p(y|x)$, i.e., $p(y|x)$ is realizable;
     \item We can compute an empirical risk minimizer, i.e., we can determine a $C_S \in \mathcal{H}$ which minimizes \eqref{eq:classifier_loss} for a given sample $S$.
\end{enumerate}
\end{assumption}
Note that the above assumptions are typical assumptions in statistical learning theory \cite{bookUML,vanHandel}.
Under these assumptions, the empirical risk minimizer converges to the desired hypothesis $C^*$ for $|S| \to \infty$. Furthermore, (deep) neural networks have an asymptotic (i.e., for increasing network capacity) universal approximation property \cite{yarotsky2018}
which makes assumption 2 fairly realistic.

\let\assumptionsref\undefined
\onevsall*
\begin{proof}
    Without loss of generality consider a single one-vs-all classifier $C(i|x,y^*)$ with $y^*\in\mathcal{Y}$ fixed and define $\bar{y^*}:=\mathcal{Y}\setminus\{y^*\}$ as the counter-part class of class $y^*$ (class ``not $y^*$'').\\
    If we now sample $ S_{y^*}\sim \tilde{p}_{y^*}(x,y)=p(x|y)\tilde{p}_{y^*}(y)$, with $\tilde{p}_{y^*}(y^*)=\frac{1}{2}$ and $\tilde{p}_{y^*}(y)=\frac{1}{2(n-1)},\,\forall y\in\mathcal{Y}\setminus\{y^*\}$, we are weighting $y^*$ and $\bar{y^*}$ equally. The loss contribution of $C(i|x,y^*)$ in \cref{eq:classifier_loss} then becomes
    \begin{align}
        &\frac{1}{|S_{y^*}|}\sum_{(x,y)\in S_{y^*}} -\mathbbm{1}_{\{y=y^*\}}\log(C(i|x,y^*)) - \frac{1}{n-1}\mathbbm{1}_{\{y\neq y^*\}}\frac{\frac{1}{2}}{\frac{1}{2(n-1)}}\log(C(o|x,y^*))\\
        =&\frac{1}{|S_{y^*}|}\sum_{(x,y)\in S_{y^*}} -\mathbbm{1}_{\{y=y^*\}}\log(C(i|x,y^*)) - \mathbbm{1}_{\{y\neq y^*\}}\log(C(o|x,y^*))\, ,
    \end{align}
    which is the binary cross entropy loss for equal class weights of $y^*$ and $\bar{y^*}$. This shows that our chosen loss function \eqref{eq:classifier_loss} yields a balanced one-vs-all classifier. As the change in sampling from the classes does not affect $p(x|y)$, by assumptions 1--3 we obtain for $|S| \to \infty$ that
    \begin{equation}
        C(i|x,y)\longrightarrow \frac{p(x|y)}{p(x|y)+p(x|\bar{y})},\quad\forall y\in\mathcal{Y}\,.
    \end{equation}
    This implies that the re-weighted classifier $\tilde{C}$ for $|S| \to \infty $ fulfills
    \begin{align}
        \tilde{C}(i|x,y)=\frac{\frac{1}{n}C(i|x,y)}{\frac{1}{n}C(i|x,y)+\frac{n-1}{n}C(o|x,y)} \longrightarrow\, & \frac{\frac{\frac{1}{n}p(x|y)}{p(x|y)+p(x|\bar{y})}}{\frac{\frac{1}{n}p(x|y)+\frac{n-1}{n}p(x|\bar{y})}{p(x|y)+p(x|\bar{y})}}\\
        =\, & \frac{\frac{1}{n}p(x|y)}{\frac{1}{n}p(x|y)+\frac{n-1}{n}p(x|\bar{y})}\\
        =\, & \frac{p(x|y)p_U(y)}{p(x|y)p_U(y)+p(x|\bar{y})p_U(\bar{y})}\\
        =\, & \frac{p(x|y)p_U(y)}{p_U(x)}\\
        =\, & p_U(y|x) \, .
    \end{align}
    By the preceding convergence, we obtain
    \begin{align}
        \frac{\tilde{C}(i|x,y)\hat{p}(y)}{\sum_{y'}\tilde{C}(i|x,y')\hat{p}(y')} \longrightarrow\, & \frac{p_U(y|x)p(y)}{\sum_{y'} p_U(y'|x)p(y')}\\
        =\, & \frac{\frac{p_U(x|y)p_U(y)p(y)}{p_U(x)}}{\sum_{y'}\frac{p_U(x|y')p_U(y')p(y')}{p_U(x)}}\\
        =\, & \frac{p(x|y)p(y)}{\sum_{y'}p(x|y')p(y')}\\
        =\, & p(y|x)\, ,
    \end{align}
    for $|S| \to \infty $, which concludes the proof.
\end{proof}





\section{Low Dimensional Regularizer}\label{sec:regularizer_explained}

\begin{figure}
    \centering
    \includegraphics[width=0.75\textwidth,keepaspectratio]{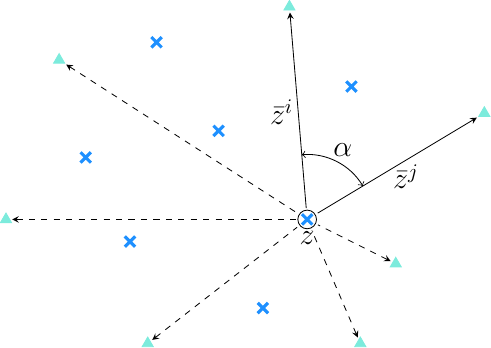}
    \caption{Illustration of the low dimensional regularizer for a single class and example (c.f.\ \cref{eq:singleregloss}).}
    \label{fig:regularizer}
\end{figure}

In this section we explain how our low dimensional regularizer behaves on a simple example, aiming at providing the reader with a clearer intuition on the role of that regularizer.
Recalling \eqref{eq:regloss}, the low dimensional regularizer applied to our GAN training is defined as
\begin{equation}\tag{\ref*{eq:regloss}}
    L_R = \frac{1}{n} \sum_{y\in\mathcal{Y}} \left[\frac{1}{N_y} \sum_{(x,y)\in S} l_R(A_{\textrm{enc}}(x,y),y)\right] \, ,
\end{equation}
with
\begin{equation}\tag{\ref*{eq:singleregloss}}
    l_R(z,y)=\frac{2}{N_y^z\cdot(N_y^z-1)}\cdot \sum_{\substack{\bar{z}^i, \bar{z}^j \in \mathcal{Z}(z,y)\\ i< j}} -\log\left(\arccos\left(\frac{\bar{z}^i\ast\bar{z}^j}{\lVert\bar{z}^i\rVert\cdot\lVert\bar{z}^j\rVert}\right)\cdot\frac{1}{\pi}\right) \, ,
\end{equation}
where for a given example $(x,y)\in S$, $z=A_{\textrm{enc}}(x,y)$ is the latent embedding of $x$, $\mathcal{Z}(z,y)=\{\tilde{z}-z|\tilde{z}=G(e,y),e\sim U(0,1)\}=\left\{\bar{z}^1,\ldots,\bar{z}^{N_y^z}\right\}$ are all, by the GAN, generated latent codes but normalized to origin $z$ and $N_y=|\{(x,y')|y'=y\}|$ are the number of examples $x$ with class label $y$.

As the regularizer loss is class conditional let us consider a simple example with a single class. In \cref{fig:regularizer} you can see an illustration similar to the left Gaussian in \cref{fig:toy_example}. In this example the cosine similarity between $\bar{z}^i$ and $\bar{z}^j$ can be computed as
\begin{equation}
    \cos(\alpha)=\frac{\bar{z}^i\ast\bar{z}^j}{\lVert\bar{z}^i\rVert\cdot\lVert\bar{z}^j\rVert}\in[-1,1]\, ,
\end{equation}
where $\cos(\alpha)=-1$ implies that $\bar{z}^i$ and $\bar{z}^j$ point in opposite directions and $\cos(\alpha)=1$ implies that they point into the same direction. In order to get a distance measure that is taking values in $[0,1]$, we compute the angle $\alpha$ in radians and normalize the result by $1/\pi$, i.e.,
\begin{equation}
    \frac{\alpha}{\pi} = \arccos\left(\frac{\bar{z}^i\ast\bar{z}^j}{\lVert\bar{z}^i\rVert\cdot\lVert\bar{z}^j\rVert}\right)\cdot\frac{1}{\pi}\in[0,1]\, .
\end{equation}

This is now the same quantity as in \cref{eq:singleregloss}. By maximizing the average angular distance between all unique pairs $\bar{z}^i,\bar{z}^j\in\mathcal{Z}(z,y),\,i<j$, we are forcing the generator to spread all generated latent codes on the boundary of the distribution represented by the class-specific data example. Transforming these angular distances via a logarithm and averaging over all $(x,y)\in S$, yields our low dimensional regularizer in \cref{eq:regloss}.

\Cref{fig:samples_different_lambda_reg} shows OoC examples generated by GANs who where trained with different weights for the low dimensional regularizer ($\lambda_\textrm{reg}$). Especially in the MNIST 0-4 example one can see that the GAN trained with $\lambda_\textrm{reg}=32$ produces much more diverse samples compared to no low-dimensional regularization.

\begin{figure}
    \centering
    \setlength{\tabcolsep}{3pt}
    \begin{tabular}{ m{0.5cm} m{16cm} }
        a) & \includegraphics[width=0.9\textwidth]{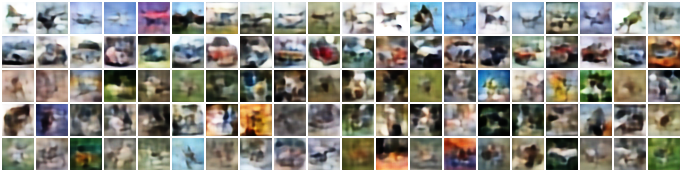} \\
        b) & \includegraphics[width=0.9\textwidth]{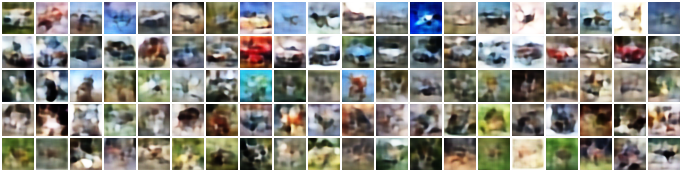}\\
        c) & \includegraphics[width=0.9\textwidth]{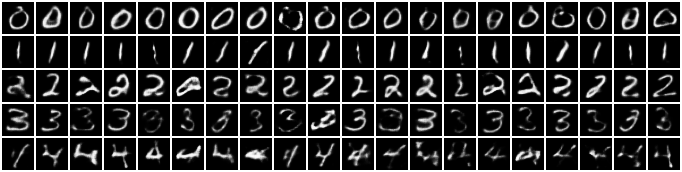}\\
        d) & \includegraphics[width=0.9\textwidth]{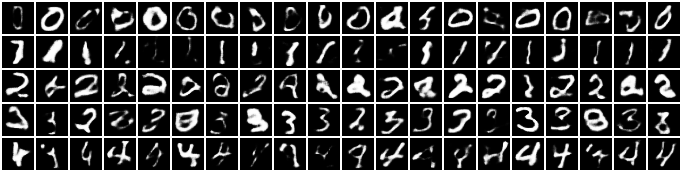}
    \end{tabular}
    \caption{Generated OoC examples by our approach with different weights on the low dimensional regularizer. a) CIFAR10 0-4 samples with $\lambda_\textrm{reg}=0$; b) CIFAR10 0-4 samples with $\lambda_\textrm{reg}=1$; c) MNIST 0-4 samples with $\lambda_\textrm{reg}=0$; d) MNIST 0-4 samples with $\lambda_\textrm{reg}=32$.}
    \label{fig:samples_different_lambda_reg}
\end{figure}

\section{Additional Toy Examples}\label{app:additional_toy}

\begin{figure}[t]
    \centering
    \includegraphics[width=0.32\textwidth]{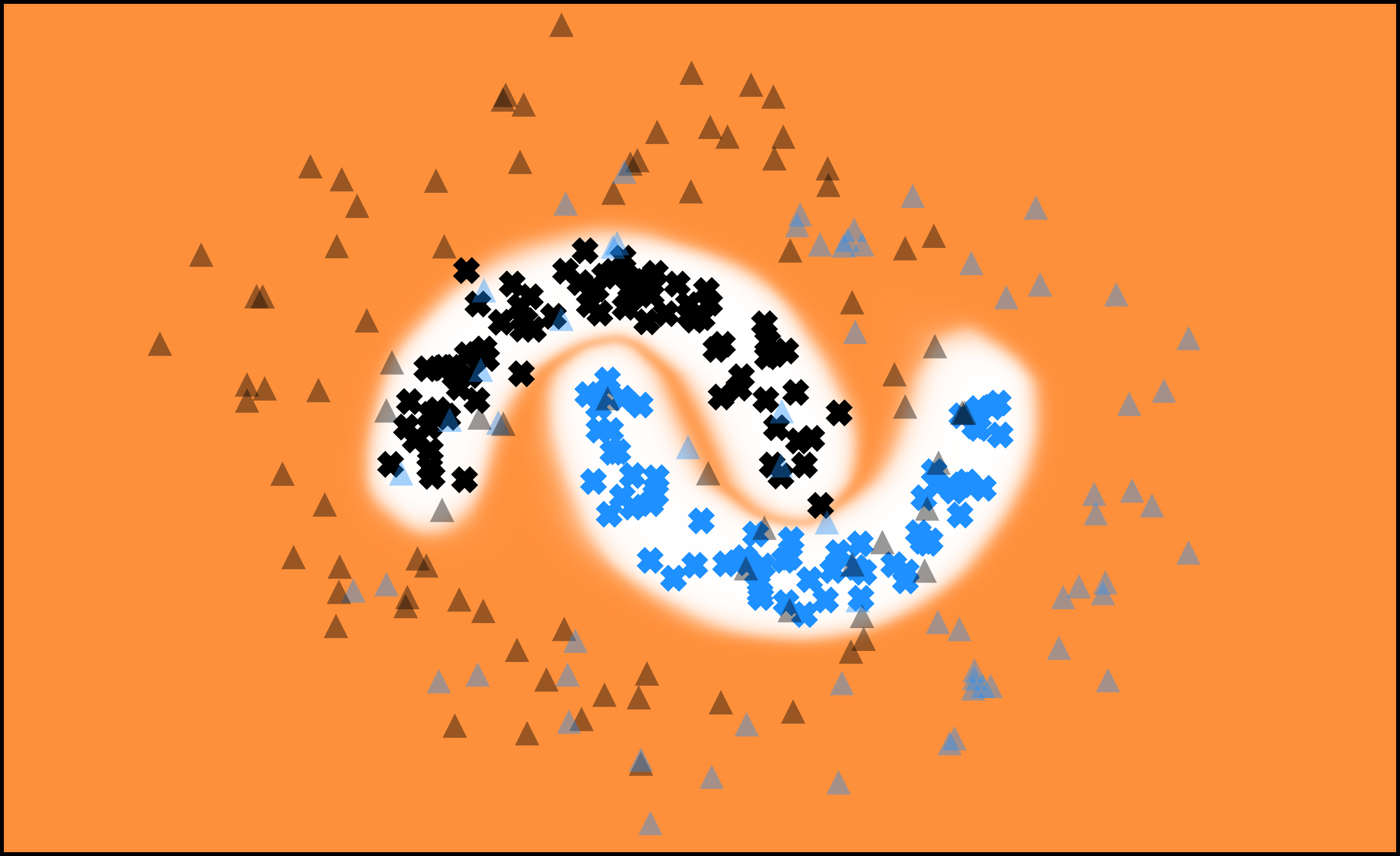}
    \includegraphics[width=0.32\textwidth]{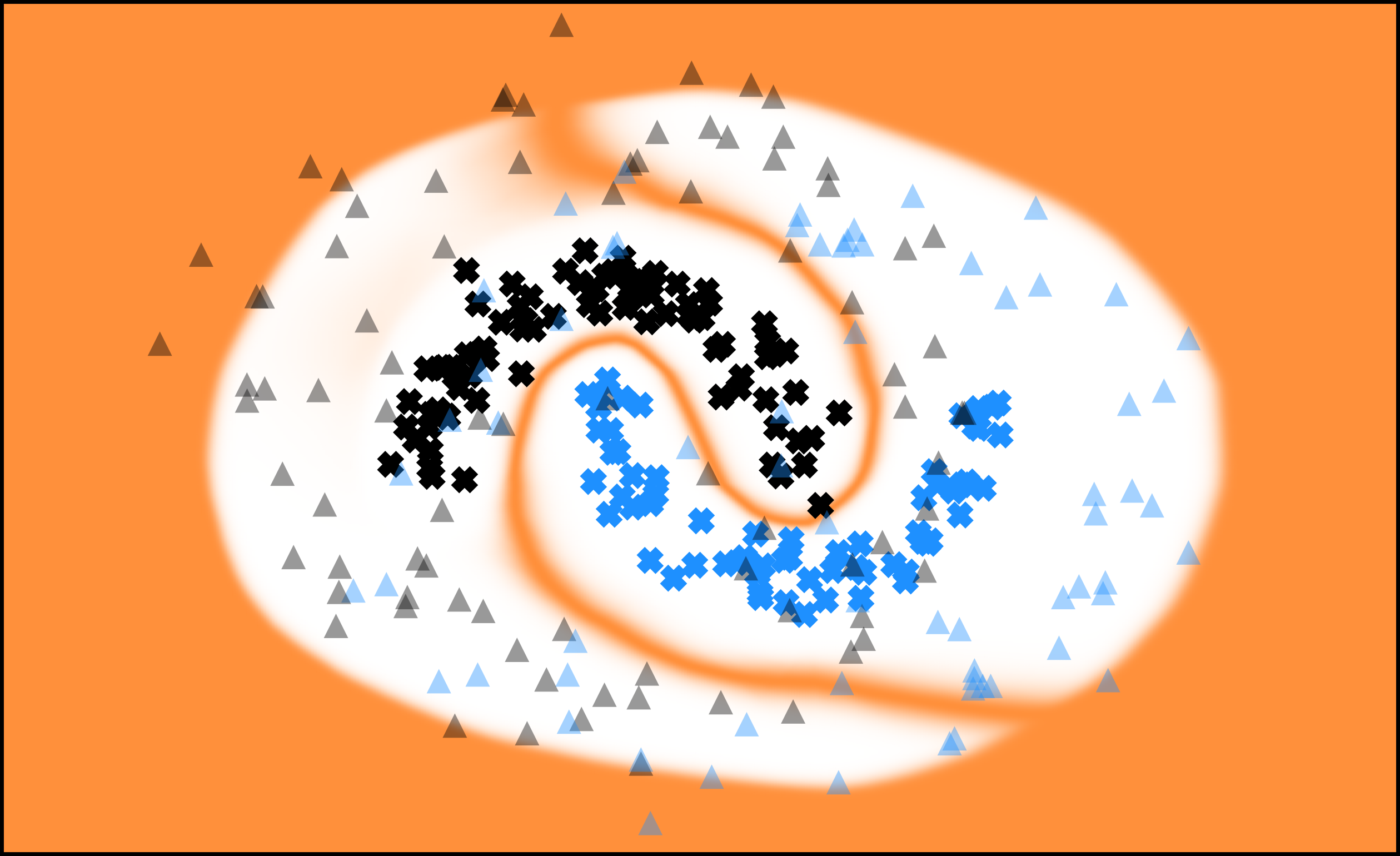}
    \includegraphics[width=0.32\textwidth]{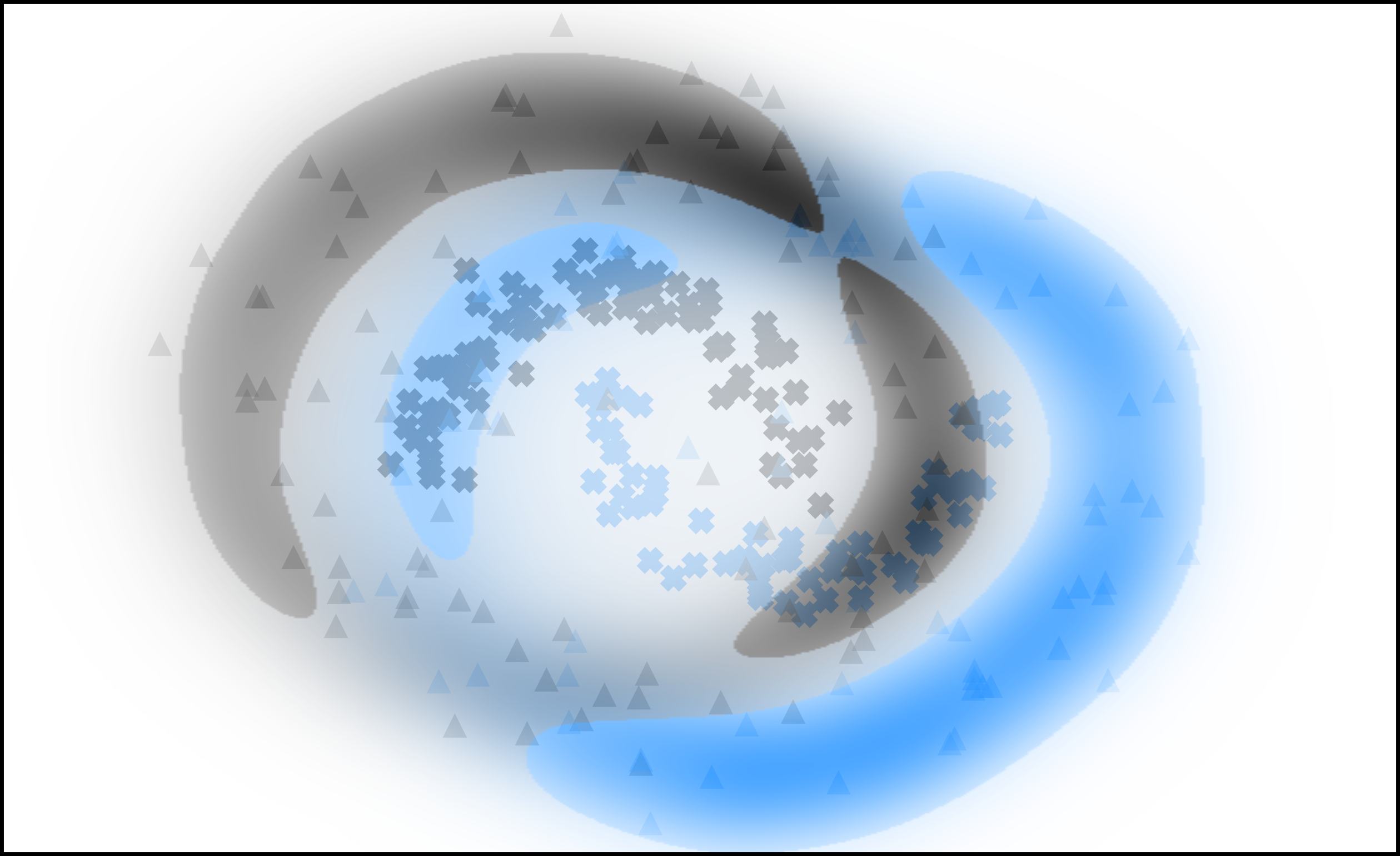}\\
    \includegraphics[width=0.32\textwidth]{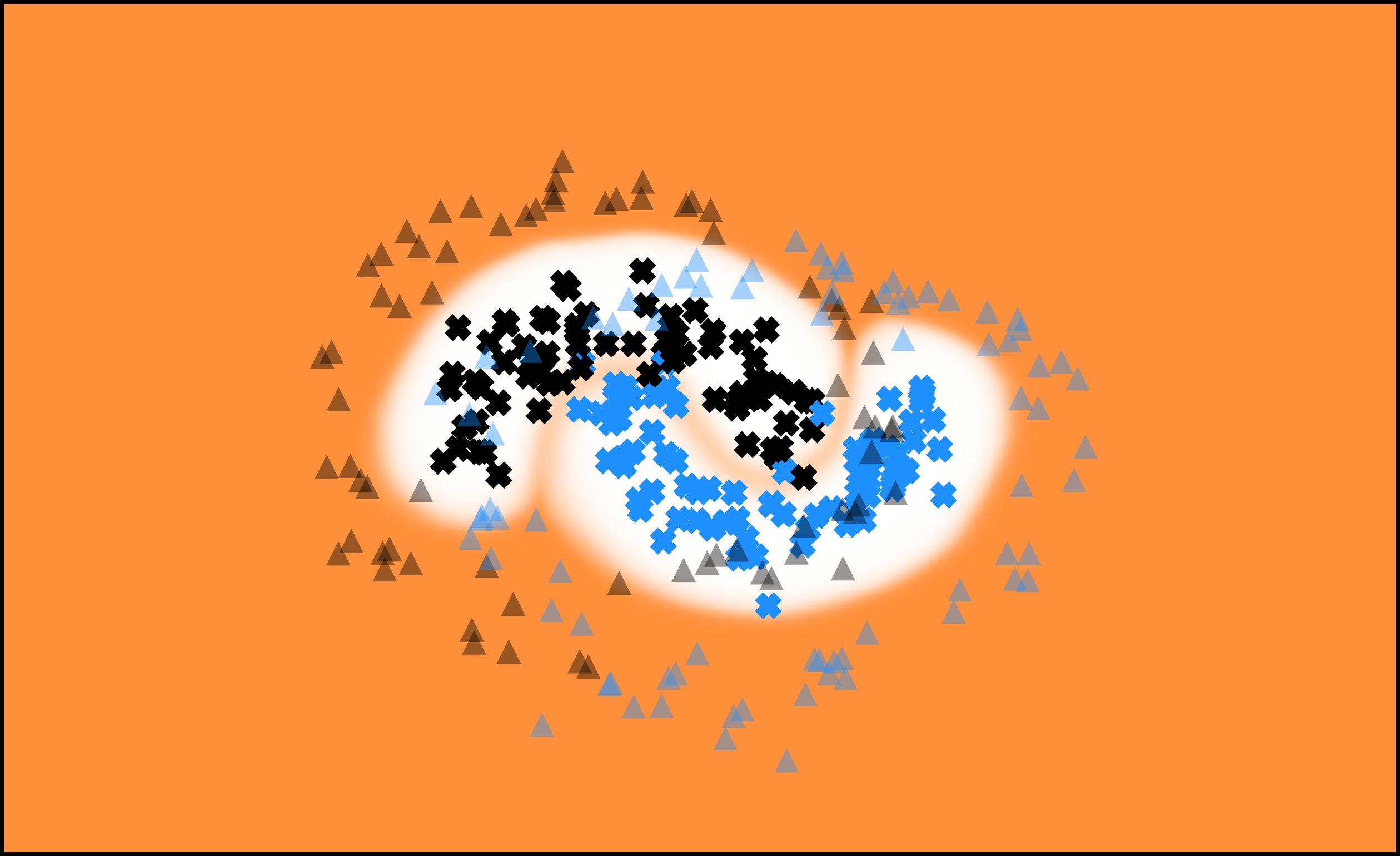}
    \includegraphics[width=0.32\textwidth]{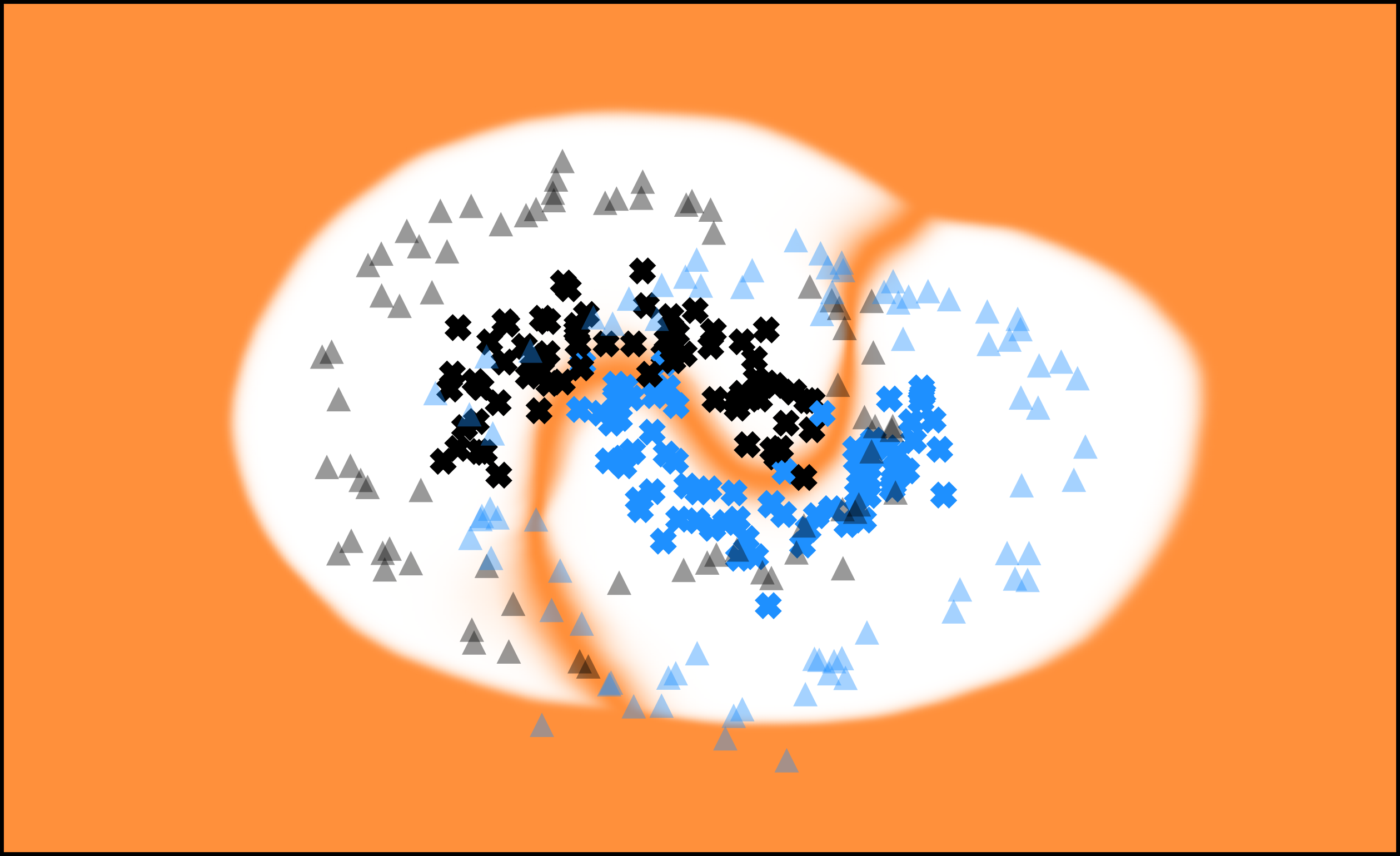}
    \includegraphics[width=0.32\textwidth]{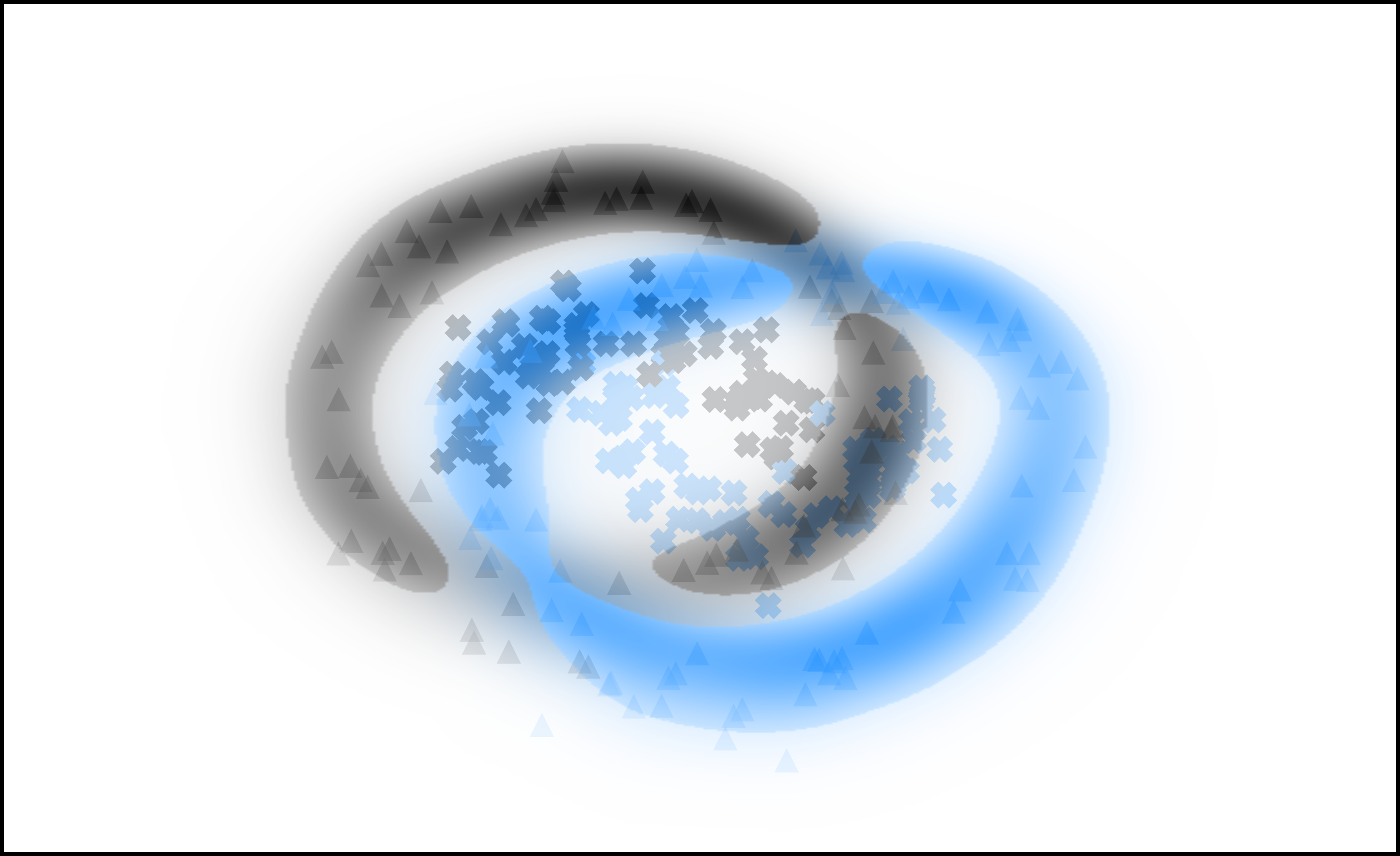}
    \caption{Two toy examples of the two moons dataset with different variance. From left to right: 1. OoD heatmap with orange indicating a high probability of being OoD and white for in-distribution; 2. Aleatoric uncertainty (entropy over \Cref{eq:classifier}) with orange indicating high and white low uncertainty; 3. Gaussian kernel density estimate of the GAN examples. Triangles indicate GAN OoC examples and crosses correspond to the in-distribution data. The data underlying the bottom row has higher variance than the one underlying the top row.}
    \label{fig:toy_example_two_moons}
\end{figure}
\begin{figure}[t]
    \centering
    \includegraphics[width=0.45\textwidth]{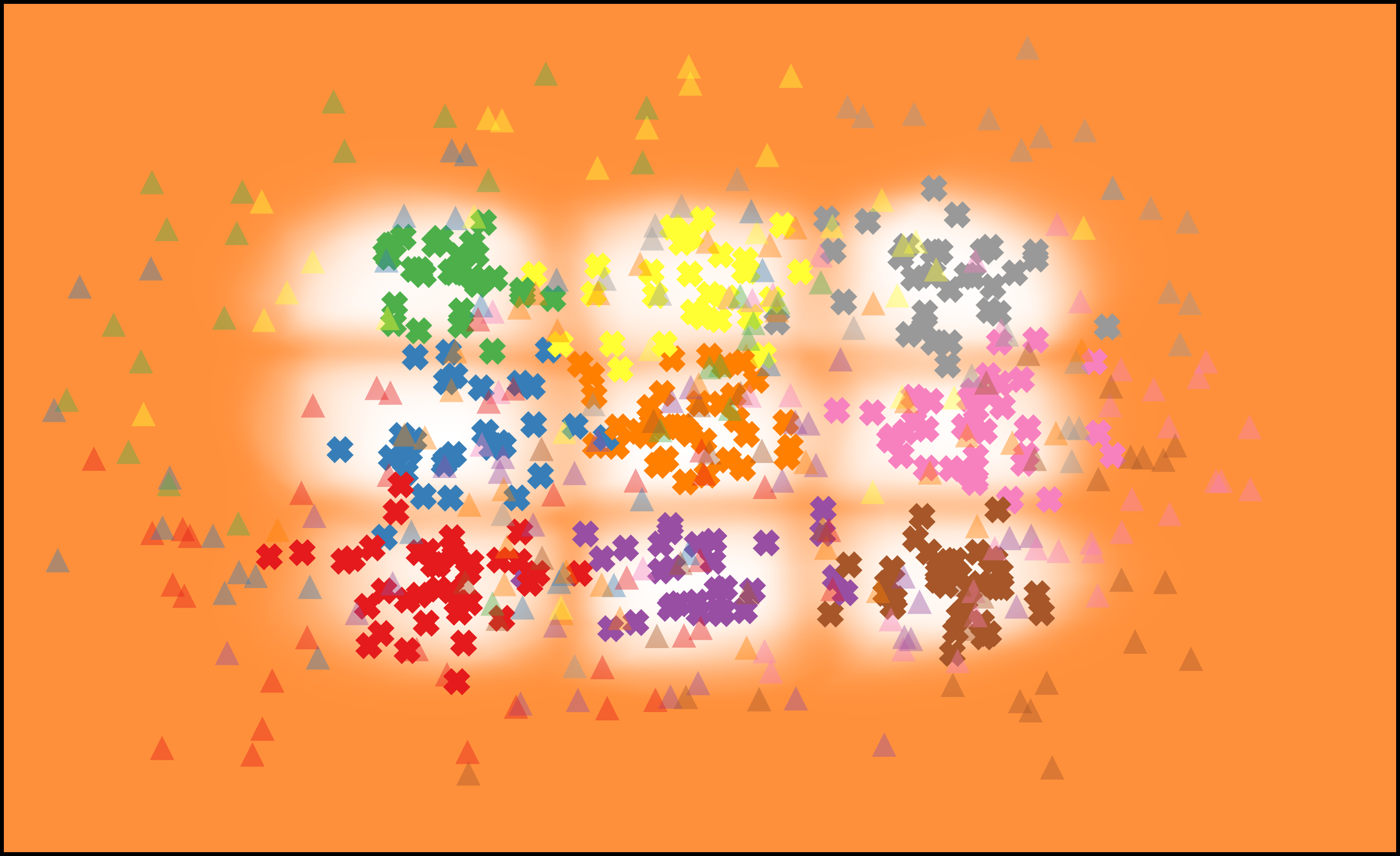}
    \includegraphics[width=0.45\textwidth]{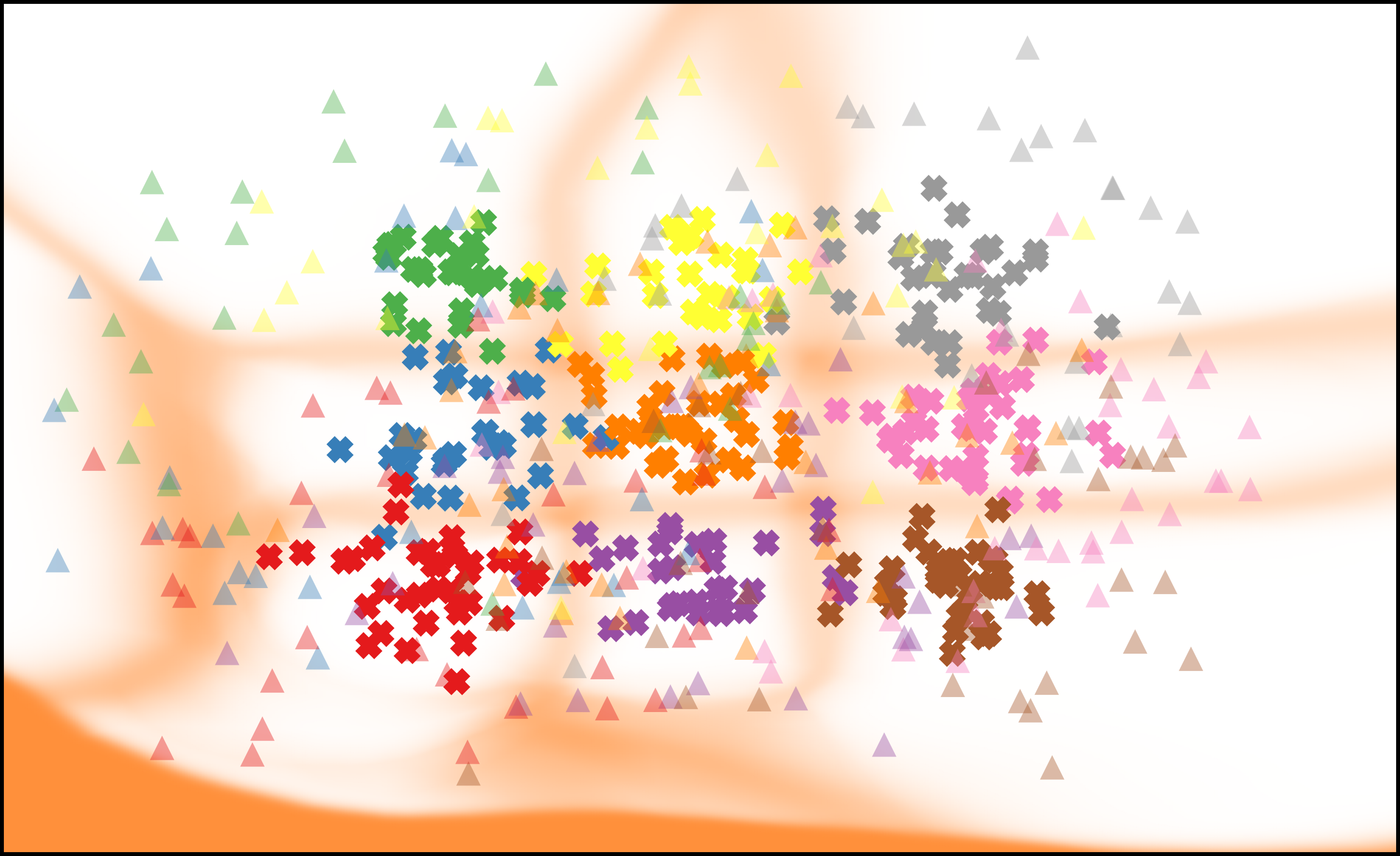}\\
    \includegraphics[width=0.45\textwidth]{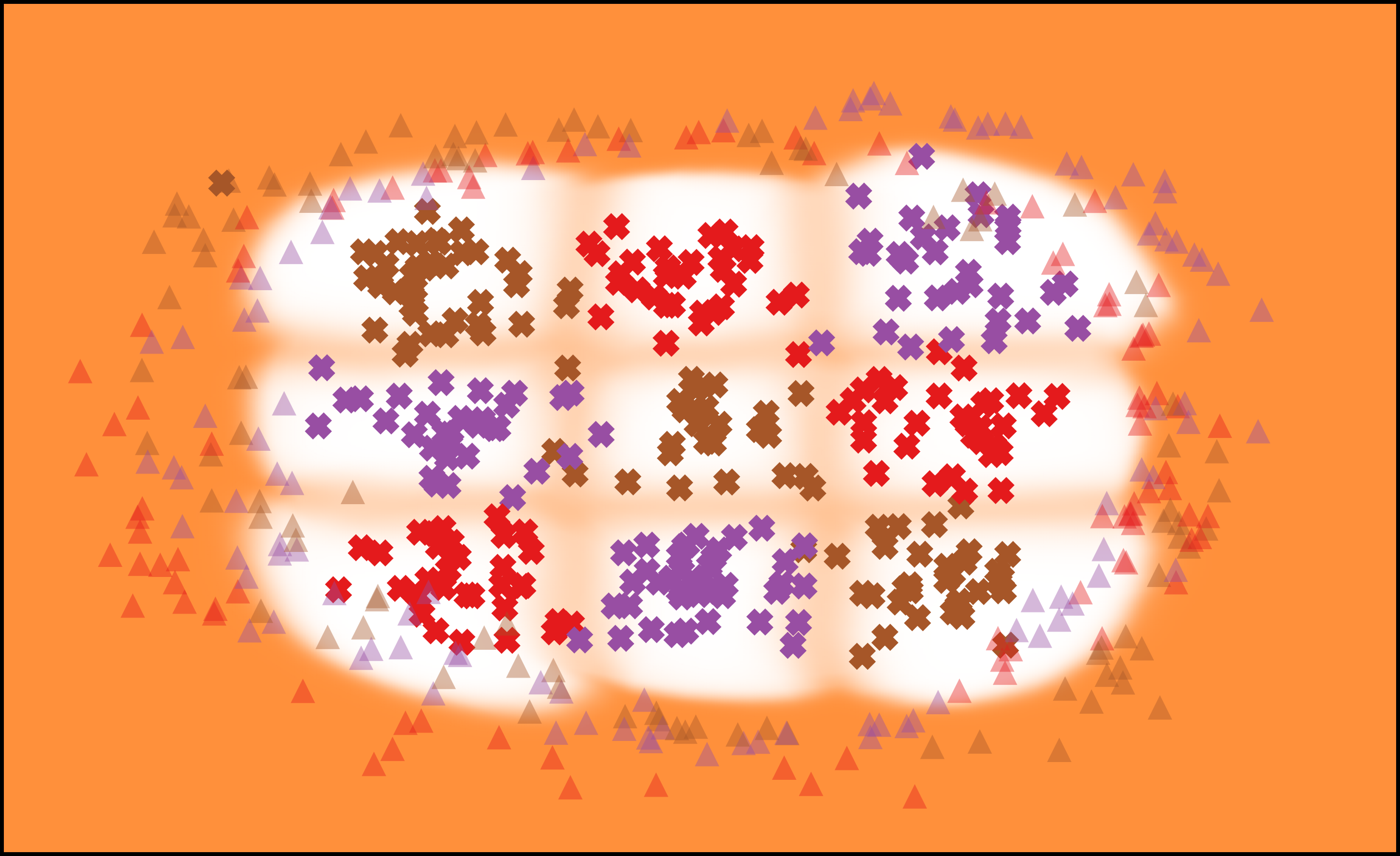}
    \includegraphics[width=0.45\textwidth]{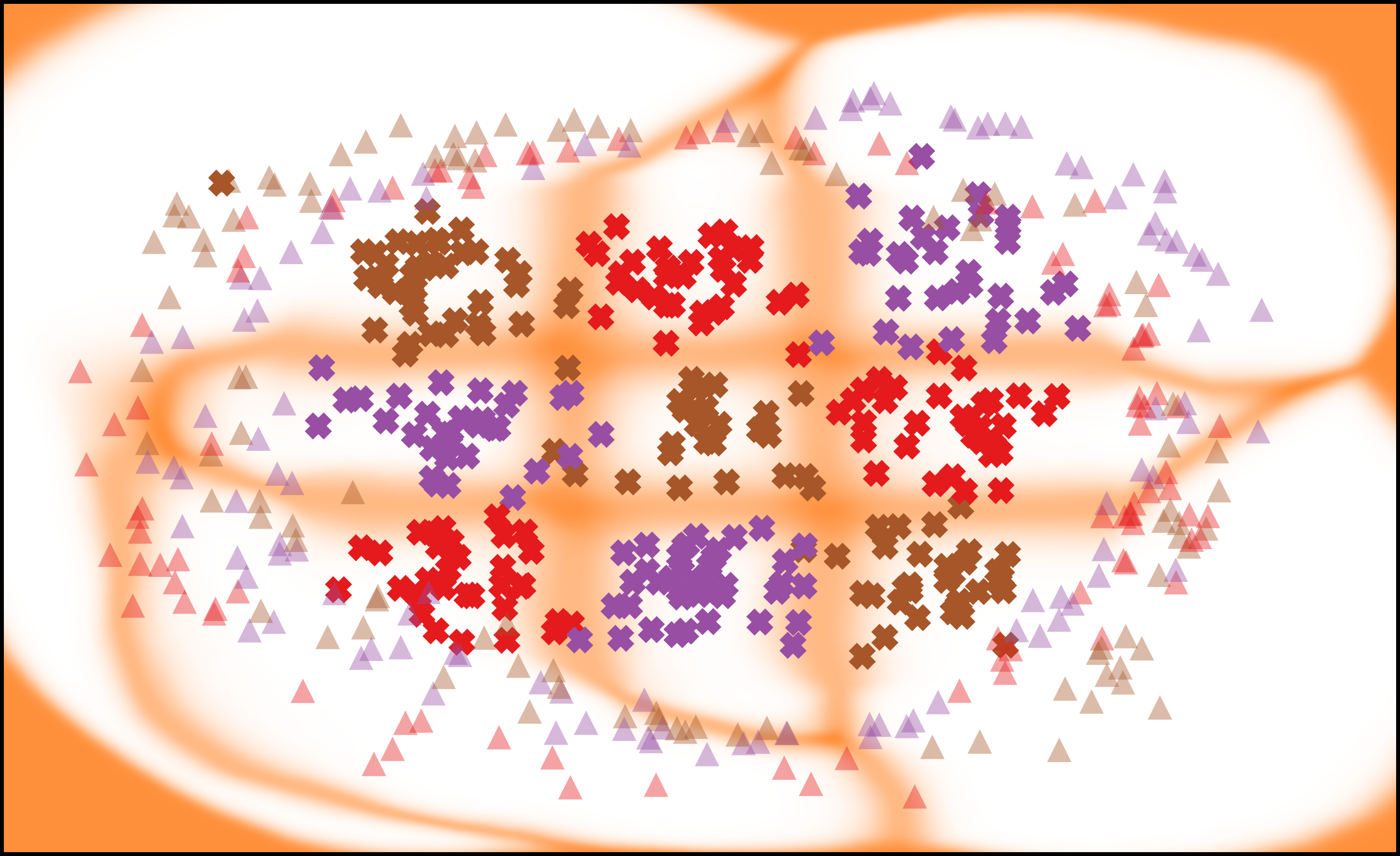}
    \caption{Two toy examples of a 3x3 grid of independent Gaussians. In the top row each Gaussian has its own class assigned, resulting in 9 classes and in the bottom row disconnected Gaussians were assigned to 3 classes in total. From left to right: 1. OoD heatmap with orange indicating a high probability of being OoD and white for in-distribution; 2. Aleatoric uncertainty (entropy over \Cref{eq:classifier}) with orange indicating high and white low uncertainty; Triangles indicate GAN OoC examples and crosses correspond to the in-distribution data.}
    \label{fig:toy_example_3x3}
\end{figure}

As a more challenging 2D example, we also present a result on the two moons dataset. In  \cref{fig:toy_example_two_moons}, the results of an experiment with two separable classes is shown. In the top row of the figure, the training data is class-wise separable. In that case, we observe that the decision boundary also belongs to the OoD regime (top left panel), which is true as there is no in-distribution data present. Our model is able to learn this since the classes are shielded tightly enough such that the generated OoC examples are in part also located in the vicinity of the decision boundary. To better visualize the distribution of generated data, we depict estimated densities of the generated OoC data in the top right panel.
For the aleatoric uncertainty in the top center panel, we observe that due to numerical issues, aleatoric uncertainty increases further away from the in-distribution data. However this can be accounted for by first considering epistemic uncertainty and then the aleatoric one. By this procedure, most of the examples close to the decision boundary would be correctly classified as OoD which is also correct since there is only a minor amount of aleatoric uncertainty involved in this example due to moderate sample size not being reflected by the data.\\
In the bottom row example, an experiment analogous to the top row but with a noisier version of the data is presented. The bottom left panel shows that the epistemic uncertainty on the decision boundary between the two classes clearly decreases in comparison to the top left panel. At the same time the bottom center panel shows the gain in aleatoric uncertainty compared to the top center panel. Note that for data points far away from the in-distribution regime, all $C(i|x,y)$ take values close to zero. In practice this requires the inclusion of a small $\varepsilon>0$ in the denominator to circumvent numerical problems in the logarithmic loss terms. In our experiments this also results in a high aleatoric uncertainty far away from the in-distribution as all estimated probabilities uniformly take the lower bound's value $\varepsilon$. However, a joint consideration of aleatoric and epistemic uncertainty disentangles this, since a high estimated probability of being OoD means that the estimate of aleatoric uncertainty can be neglected. This also becomes evident in the center panels where one can observe high aleatoric uncertainty $H(x)$ outside the in-distribution regime, which can however be masked out by the OoD probability $\tilde{C}(o|x)$.

\Cref{fig:toy_example_3x3} shows a toy example on a $3\times 3$ grid of Gaussians. In the top row each Gaussian belongs to its own class, resulting in $9$ classes in total, while in the bottom row multiple Gaussians belong to one class, resulting in disconnected class regions. In both cases our method is able to predict a high aleatoric uncertainty between class boundaries and high epistemic uncertainty away from the training data. We can also observe the high aleatoric uncertainty far away from the training data as in \cref{fig:toy_example_two_moons}, which can also be disentangled by a joint consideration of aleatoric and epistemic uncertainty.

\section{Hyperparameter Settings for Experiments}\label{app:hyperparams}

For the Bayes-by-Backprop implementation we use \textit{spike-and-slap} priors in combination with diagonal Gaussian posterior distributions as described in \cite{BlundellCKW15}. MC-Dropout uses a 50\% dropout probability on all weight layers. Both mentioned methods average their predictions over 50 forward passes. The deep-ensembles were built by averaging 5 networks. Implementations of Confident Classifier and GEN use the architectures and hyperparameters recommended by the authors and we followed their reference code where possible. Parameter studies showed that our method is mostly stable w.r.t.\ the hyperparameter selection. We used $\lambda_{gp}=10$ as proposed in \cite{GulrajaniAADC17}, $\lambda_{cl}=2$ for MNIST/Tiny Imagenet and $\lambda_{cl}=4$ for CIFAR10/100, $\lambda_{\textrm{real}}=0.6$, $\lambda_R=32$ for MNIST/Tiny ImageNet and $\lambda_R=1$ for CIFAR10/100.

The latent dimension for MNIST was set to $32$ while using $128$ dimensions for CIFAR10/100 and Tiny ImageNet. We used batch stochastic gradient descent with the ADAM \cite{adam} optimizer and a batch size of $256$. The learning rate was initialized to $10^{-3}$ for the classification model and $2\cdot 10^{-4}$ for the GAN while linearly decaying them to $10^{-5}$ over the course of all training iterations. Training on the MNIST dataset required $2\,000$ generator iterations while taking $10\,000$ iterations for the CIFAR10/100 and Tiny ImageNet datasets (one iteration is considered to be one gradient step on the generator). As recommended in \cite{GulrajaniAADC17}, we use batch normalization only in the generator, while the critic as well as the classifier do not use any type of layer normalization. We also adopt the alternating training scheme from the just mentioned work. For each generator iteration the critic as well as classifier are performing $5$ optimization steps on the same batch. We do apply some mild data augmentation by using random horizontal flipping where appropriate. The test set sizes used for computing the numerical results can be found in \cref{tab:testset_sizes}. Besides the already mentioned LeNet and ResNet architectures for the classification model, the cAE uses a small convolutional architecture for MNIST and a ResNet architecture for CIFAR10/100 and Tiny ImageNet. The generator and discriminator are both implemented as fully connected networks with $(1024, 512, 256)$ neurons for the generator and $(512, 512, 512)$ neurons for the discriminator. We experimented with different sizes of the generator and observed low to no influence on the final performance of the classifier. All experiments were performed on a Nvidia RTX 3090, Tesla P100 or A100 GPU but models with less VRAM are also sufficient as the cGAN itself is very small. On the RTX 3090 training takes approximately $50$ minutes, $3$ hours, $3.5$ hours and $5.5$ hours for MNIST, CIFAR10, CIFAR100 and Tiny ImageNet, respectively.
\begin{table}[htb]
    \centering
    \caption{Test set sizes used for computing the numerical results.}
    \label{tab:testset_sizes}
    \scalebox{0.9}{
    \begin{tabular}{cr}
        \toprule
         Dataset & Test-Set Size \\
         \midrule
         MNIST 0-4 & $5\,000$\\
         MNIST 5-9 & $5\,000$\\
         CIFAR10 0-4 & $5\,000$\\
         CIFAR10 5-9 & $5\,000$\\
         Tiny ImageNet 0-99 & $5\,000$\\
         Tiny ImageNet 100-199 & $5\,000$\\
         EMNIST-Letters & $20\,800$\\
         Fashion-MNIST & $10\,000$\\
         SVHN & $26\,032$\\
         Omniglot & $13\,180$\\
         LSUN & $10\,000$\\
         \bottomrule
    \end{tabular}
    }
\end{table}

\section{Parameter Study}\label{app:parameter_study}

\begin{figure}
    \centering
    \includegraphics[width=\textwidth]{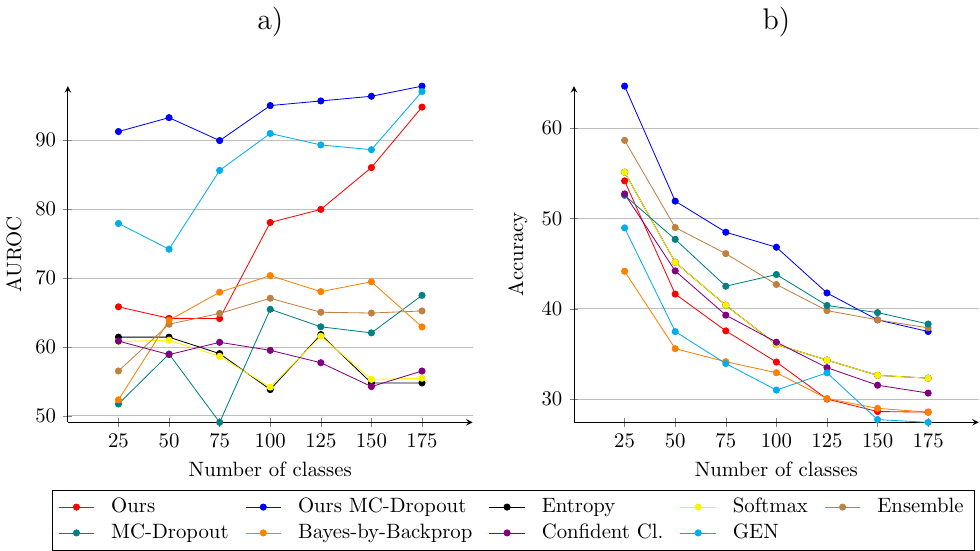}
    \caption{Influence of the number of classes on the OoD detection and accuracy based on the Tiny ImageNet dataset. a) Influence on the AUROC; b) Influence on the accuracy. The number of classes corresponds to the Tiny ImageNet dataset and the rest of the classes are then considered as OoD. Additonally the SVHN, Fashion-MNIST and MNIST datasets were used as OoD. Hyperparameters and models are the same as for the results in \cref{tab:res_tinyimagenet,tab:res_tinyimagenet_breakdown} and kept fixed for different numbers of classes.}
    \label{fig:ablation_number_of_classes}
\end{figure}

In order to examine the influence of hyperparameter selection onto our framework we conducted an extensive experimental study. We displayed all mentioned evaluation metrics from \cref{sec:experiments} while varying $\lambda_\textrm{cl}$, $\lambda_\textrm{real}$, $\lambda_\textrm{reg}$ and the chosen latent dimension for the cAE and cGAN. For $\lambda_\textrm{cl}$ in \cref{eq:gan_complete_gd}, which controls the influence of the classifier predictions onto the generated OoC examples, \cref{fig:lambda_cl} shows clearly that for MNIST $\lambda_\textrm{cl}=2$ and for CIFAR10 $\lambda_\textrm{cl}=4$ are locally optimal values w.r.t.\ maximum performance. While larger $\lambda_\textrm{cl}$ tend to increase the in-distribution accuracy slightly it greatly decreases all other evaluation metrics. A very interesting observation in \cref{fig:lambda_real} about the influence of $\lambda_\textrm{real}$ from \cref{eq:gan_complete_c} is that both extremes ($\lambda_\textrm{real}\in\{0,1\}$) are greatly decreasing the results. This shows clearly that the generated OoC examples have a positive effect on the OoD detection performance and in-distribution separability. In terms of the dimensionality of the latent space, \cref{fig:latent_dim} shows that $32$ and $128$ dimensions are the optimal values for MNIST and CIFAR10, respectively. This is coherent with the visual quality of the examples decoded by the cAE, which does not improve much with higher dimensions. Analysing the influence of $\lambda_\textrm{reg}$ in \cref{fig:lambda_reg} one can observe a positive effect on the results on the MNIST dataset by increasing the value of the parameter up to $\lambda_\textrm{reg}=32$ where we reach a local maximum. It is also very apparent, that for $\lambda_\textrm{reg}=0$ the results are comparatively bad, emphasizing the role of the low dimensional regularizer in our model. For CIFAR10 the effect is not as clear as for the MNIST dataset, but \cref{fig:lambda_reg} also shows a locally optimal setting of $\lambda_\textrm{reg}\in[1,2]$. We believe that the higher latent dimension required for the CIFAR10 dataset and thus the curse of dimensionality is the main factor behind this finding. \Cref{fig:samples_different_lambda_reg} shows qualitatively the influence of $\lambda_\textrm{reg}$ on the generated OoC examples. We can observe that for a training with a higher $\lambda_\textrm{reg}$ we obtain much more diverse examples compared to when using no regularizer. This is especially apparent in the MNIST setting.

\Cref{fig:ablation_number_of_classes} shows the OoD detection performance w.r.t.\ AUROC and the accuracy on the in-distribution set while increasing the number of classes present during training. The hyperparameters as well as all model architectures were kept fixed for this experiment. We used the Tiny ImageNet dataset with varying class-wise subsets as in-distribution data and the remaining Tiny ImageNet classes, SVHN, Fashion-MNIST and MNIST datasets as OoD examples. All methods were trained on our training subset while the figure displays the results on the respective test sets. For GEN we faced numerical problems with exploding evidence values. Due to this, we chose different hyperparameter settings for different numbers of classes, which also results in larger standard deviation values for GEN in \cref{tab:res_tinyimagenet,tab:res_tinyimagenet_breakdown}. The study shows that increasing the number of classes also increases the OoD detection performance of our approach. This is also true for GEN, which is why we argue that the gain in performance can in part be attributed to the use of an autoencoder. As the number of classes increases, so does the number of examples in the training dataset, improving the ability of the autoencoder to produce diverse embeddings. When observing the accuracy on the in-distribution data the most apparent observation is that for all methods the accuracy decreases for a higher number of classes. This is not surprising as we kept the classification model fixed. In general we observe that our approach with MC-dropout has the overall highest performance, being equal to standard MC-Dropout and Ensembles from $150$ classes and above. Our approach without dropout achieves mid-tier results while still being superior to GEN.
\begin{figure}[tb]
    \centering
    \includegraphics[width=\textwidth,height=0.4\textheight,keepaspectratio]{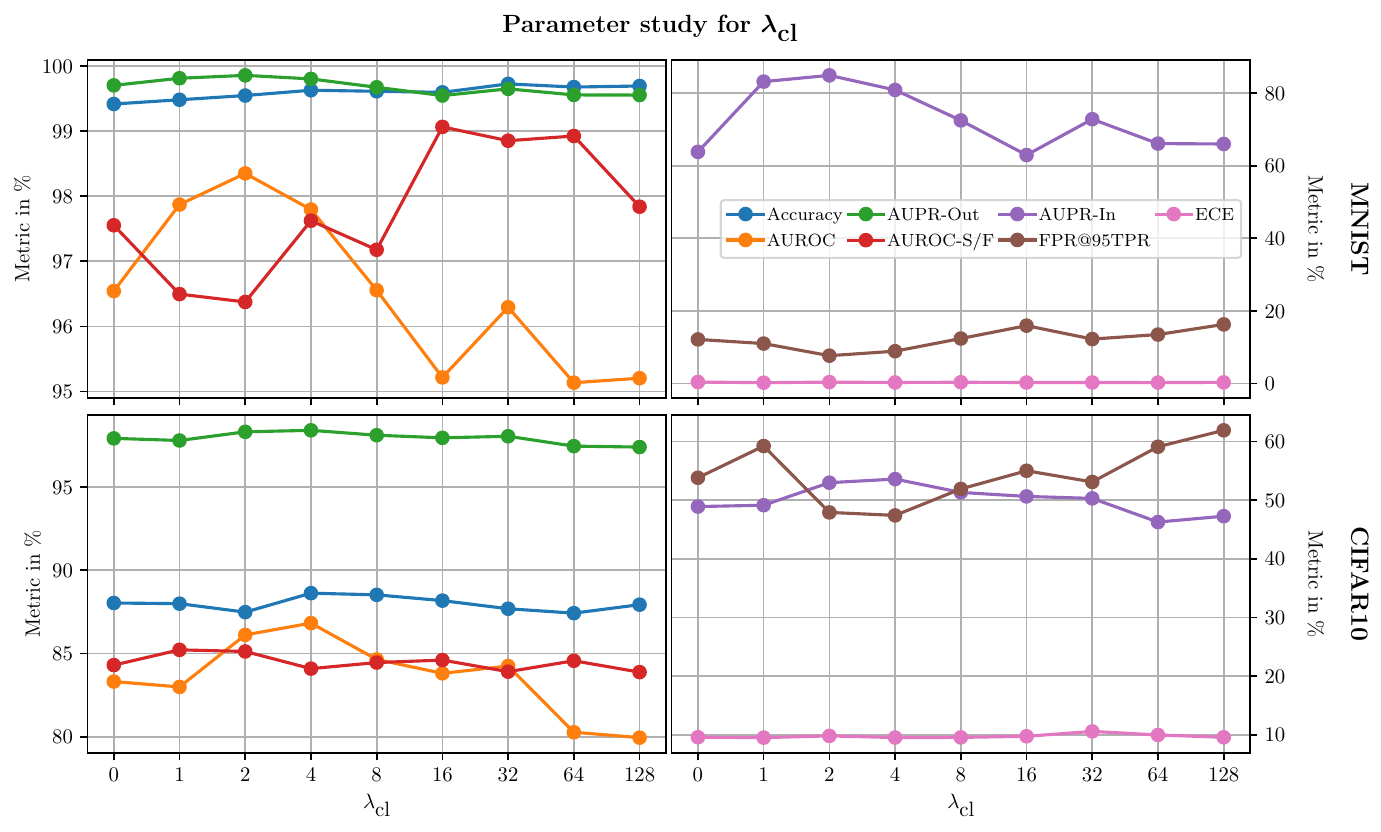}
    \caption{Parameter study over $\lambda_\textrm{cl}$ in \cref{eq:gan_complete_gd}. For MNIST hyperparameters were fixed at $\lambda_\textrm{reg}=14$, $\lambda_\textrm{real}=0.5$, $\text{latent dimension}=16$ and for CIFAR10 at $\lambda_\textrm{reg}=0$, $\lambda_\textrm{real}=0.6$, $\text{latent dimension}=128$. All seeds were also the same for all experiments. All metrics were computed on the validation sets of MNIST 0-4 / CIFAR10 0-4 as in-distribution datasets and the entirety of all assigned OoD datasets as defined in \cref{sec:experiments}.}
    \label{fig:lambda_cl}
\end{figure}
\begin{figure}[tb]
    \centering
    \includegraphics[width=\textwidth,height=0.4\textheight,keepaspectratio]{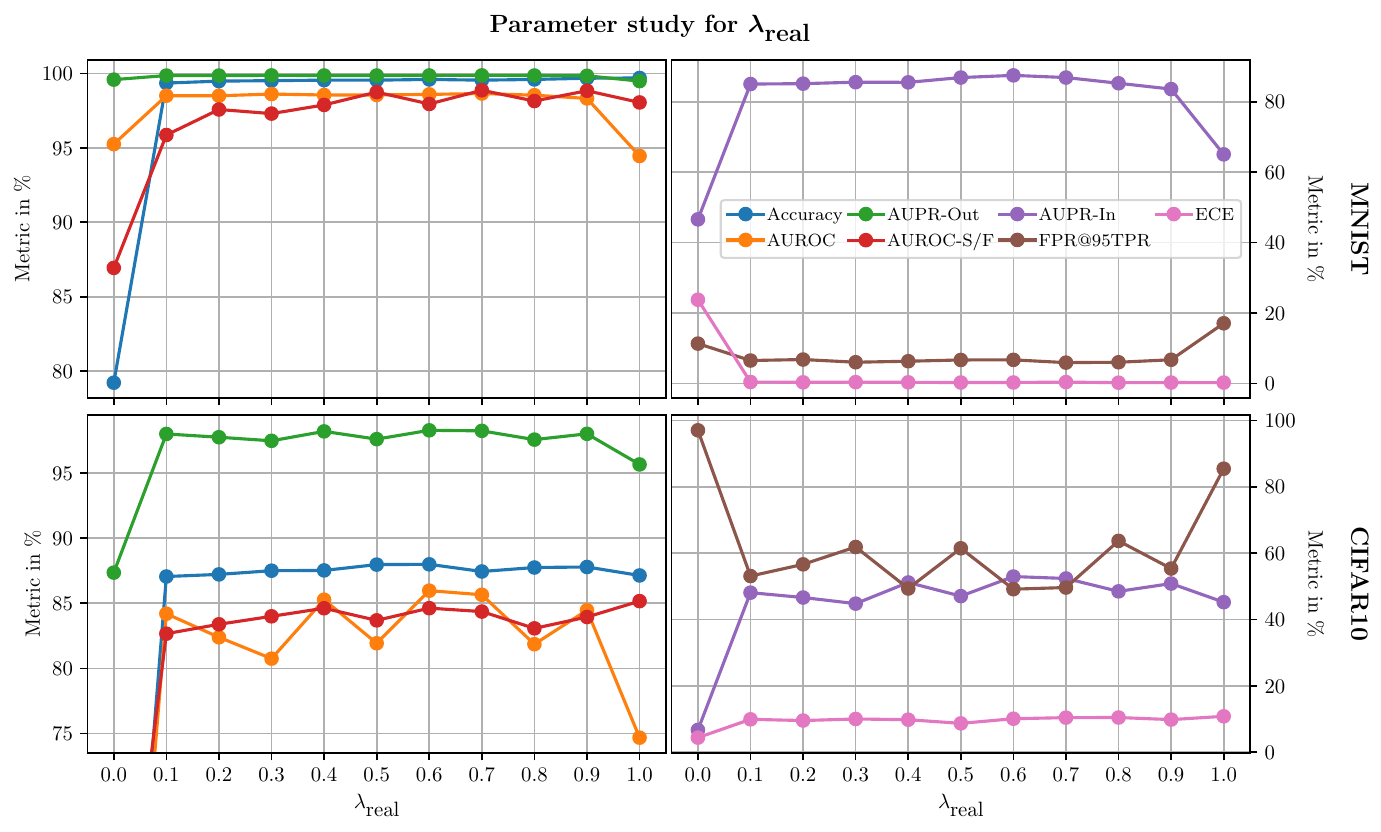}
    \caption{Parameter study over $\lambda_\textrm{real}$ in \cref{eq:gan_complete_c}. For MNIST, hyperparameters were fixed at $\lambda_\textrm{reg}=32$, $\lambda_\textrm{cl}=1$, $\text{latent dimension}=16$ and for CIFAR10 at $\lambda_\textrm{reg}=0$, $\lambda_\textrm{cl}=2$, $\text{latent dimension}=128$. All seeds were also the same for all experiments. All metrics were computed on the validation sets of MNIST 0-4 / CIFAR10 0-4 as in-distribution datasets and the entirety of all assigned OoD datasets as defined in \cref{sec:experiments}.}
    \label{fig:lambda_real}
\end{figure}
\begin{figure}[tb]
    \centering
    \includegraphics[width=\textwidth,height=0.4\textheight,keepaspectratio]{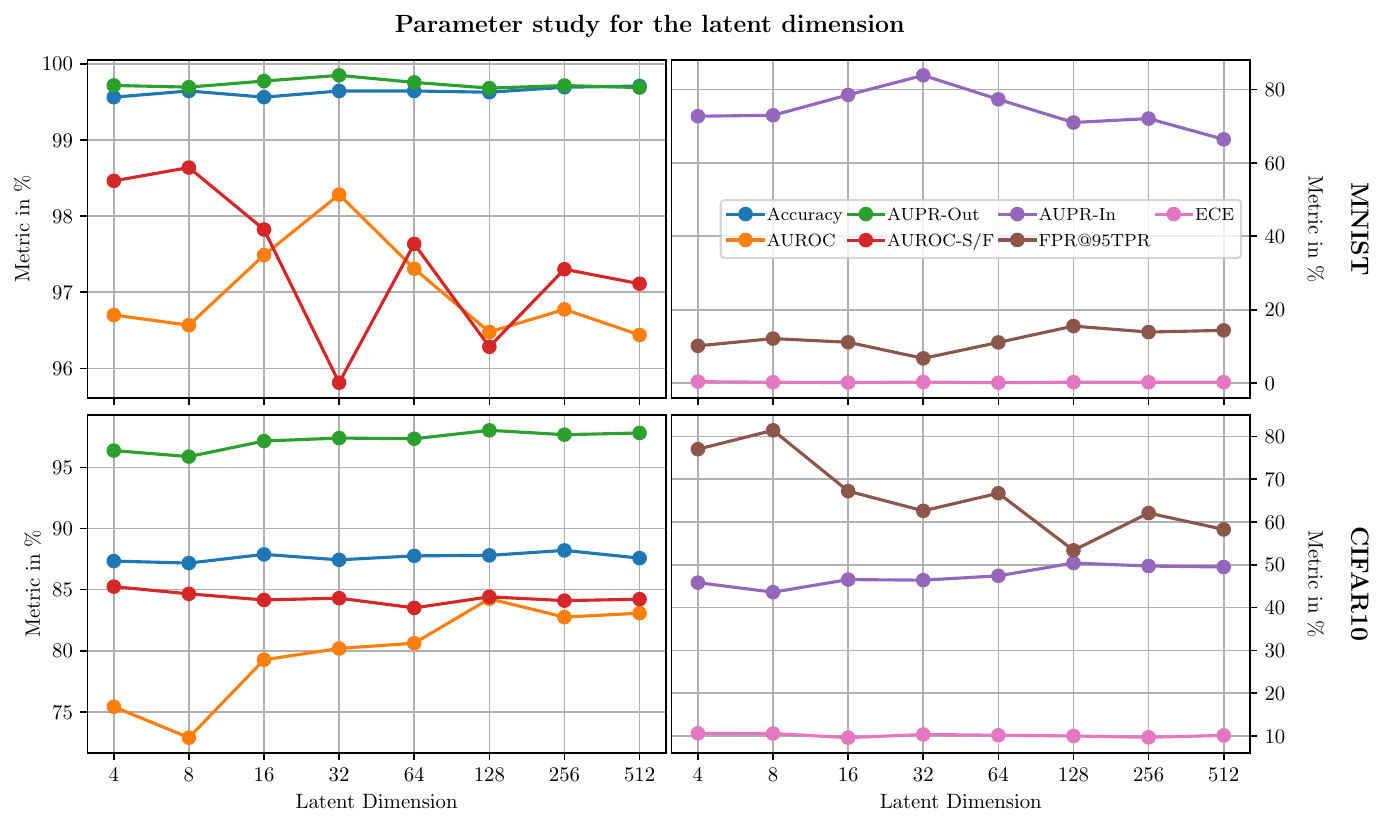}
    \caption{Parameter study over latent dimensions of $z$ in \cref{eq:gan_complete_gd}. For MNIST, hyperparameters were fixed at $\lambda_\textrm{reg}=32$, $\lambda_\textrm{cl}=1$, $\lambda_\textrm{real}=0.5$ and for CIFAR10 at $\lambda_\textrm{reg}=0$, $\lambda_\textrm{cl}=4$, $\lambda_\textrm{real}=0.6$. All seeds were also the same for all experiments. All metrics were computed on the validation sets of MNIST 0-4 / CIFAR10 0-4 as in-distribution datasets and the entirety of all assigned OoD datasets as defined in \cref{sec:experiments}.}
    \label{fig:latent_dim}
\end{figure}
\begin{figure}[tb]
    \centering
    \includegraphics[width=\textwidth,height=0.4\textheight,keepaspectratio]{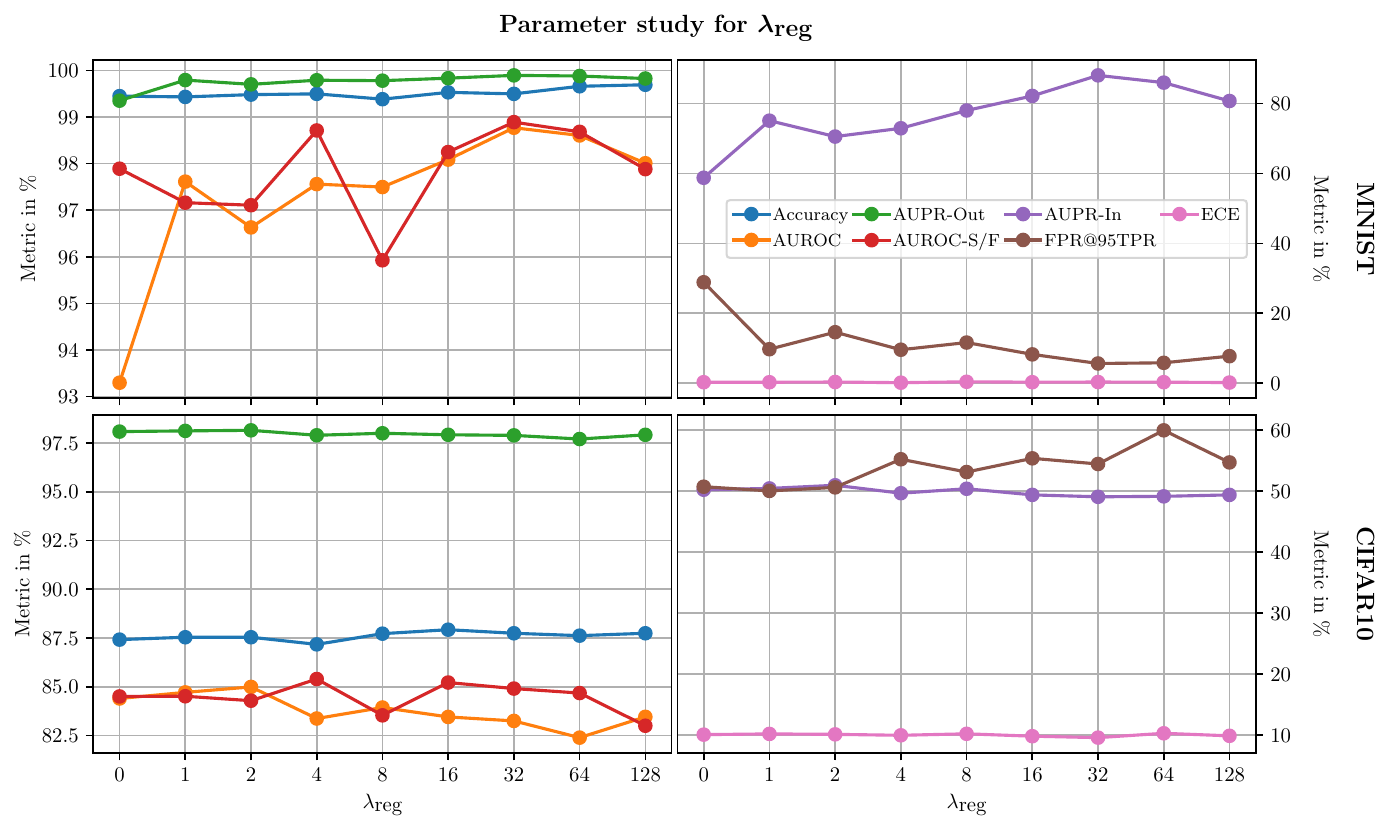}
    \caption{Parameter study over $\lambda_\textrm{reg}$ in \cref{eq:gan_complete_gd}. For MNIST, hyperparameters were fixed at $\lambda_\textrm{cl}=1$, $\lambda_\textrm{real}=0.5$, $\text{latent dimension}=16$ and for CIFAR10 at $\lambda_\textrm{cl}=4$, $\lambda_\textrm{real}=0.6$, $\text{latent dimension}=128$. All seeds were also the same for all experiments. All metrics were computed on the validation sets of MNIST 0-4 / CIFAR10 0-4 as in-distribution datasets and the entirety of all assigned OoD datasets as defined in \cref{sec:experiments}.}
    \label{fig:lambda_reg}
\end{figure}

\section{Joint Detection of OoD and FP in the Wild} \label{app:boosting}

\begin{table}[tb]
    \setlength{\tabcolsep}{3pt}
    \centering
    \caption{Results obtained from aggregating predicted uncertainty estimates in a gradient boosting model which was then trained on a validation set. In distribution dataset is CIFAR10 (0-4) and OoD datasets are CIFAR10 (5-9), LSUN, SVHN, Fashion-MNIST, MNIST.}
    \label{tab:res_boosting}
    \scalebox{\resultstablescaling}{
    \begin{tabular}{l @{\hskip 25pt} ccc @{\hskip 25pt} ccc}
        \toprule
         \multirow{2}{*}[-1em]{Method} & \multicolumn{3}{@{}c@{\hskip 25pt}}{Uncertainty Scores and Predicted Probabilities} & \multicolumn{3}{c}{Uncertainty Scores Only}\\
         & \thead{Accuracy\\TP} & \thead{Accuracy\\FP} & \thead{Accuracy\\OoD} & \thead{Accuracy\\TP} & \thead{Accuracy\\FP} & \thead{Accuracy\\OoD} \\
         \midrule
         Ours & $78.32\,(1.20)$ & $37.33\,(3.05)$ & $75.47\,(0.92)$ & $76.96\,(1.14)$ & $40.88\,(2.65)$ & $66.07\,(0.99)$\\
         Ours with MC-Dropout & $\bm{79.59\,(0.61)}$ & $47.58\,(2.68)$ & $\bm{82.50\,(0.55)}$ & $\bm{77.16\,(1.19)}$ & $41.86\,(2.92)$ & $\bm{70.37\,(0.69)}$\\
         \midrule
         One-vs-All Baseline & $68.54\,(1.65)$ & $46.54\,(1.92)$ & $73.69\,(1.81)$ & $64.78\,(2.25)$ & $33.75\,(6.83)$ & $56.02\,(2.61)$\\
         Max.\ Softmax \cite{HendrycksG17} & $66.04\,(0.33)$ & $49.51\,(1.69)$ & $71.55\,(1.48)$ & $64.98\,(0.58)$ & $33.58\,(4.01)$ & $48.77\,(4.34)$\\
         Entropy & $66.98\,(0.56)$ & $48.47\,(1.74)$ & $72.36\,(1.33)$ & $66.13\,(1.07)$ & $29.90\,(5.67)$ & $51.42\,(4.62)$\\
         Bayes-by-Backprop \cite{BlundellCKW15} & $67.98\,(0.88)$ & $44.65\,(2.05)$ & $73.34\,(0.99)$ & $64.81\,(0.93)$ & $33.62\,(4.31)$ & $52.76\,(4.31)$\\
         MC-Dropout \cite{GalG15a} & $72.20\,(0.41)$ & $\bm{52.53\,(1.29)}$ & $77.38\,(0.49)$ & $66.97\,(1.88)$ & $44.31\,(2.14)$ & $59.91\,(2.10)$\\
         Deep-Ensembles \cite{LakshminarayananPB17} & $72.10\,(0.33)$ & $47.76\,(0.90)$ & $75.30\,(0.41)$ & $67.76\,(0.56)$ & $36.29\,(1.78)$ & $53.56\,(1.01)$\\
         Confident Classifier \cite{LeeLLS18} & $68.15\,(0.77)$ & $50.64\,(1.45)$ & $72.23\,(0.71)$ & $64.69\,(1.47)$ & $39.84\,(3.85)$ & $45.62\,(3.47)$\\
         GEN \cite{SensoyKCS20} & $69.86\,(2.11)$ & $49.13\,(2.43)$ & $76.04\,(1.17)$ & $70.47\,(0.94)$ & $\bm{64.88\,(1.52)}$ & $54.31\,(2.52)$\\
         \midrule
         Entropy Oracle & $76.89\,(0.32)$ & $54.33\,(2.05)$ & $86.59\,(0.37)$ & $71.36\,(1.59)$ & $61.40\,(2.77)$ & $82.35\,(0.62)$\\
         One-vs-All Oracle & $75.35\,(0.90)$ & $58.20\,(2.03)$ & $80.00\,(2.43)$ & $73.60\,(1.74)$ & $60.49\,(0.84)$ & $74.78\,(2.85)$\\
         \bottomrule
    \end{tabular}
    }
\end{table}

In this section, we perform CIFAR10-based OoD and FP detection in the wild, i.e., we perform both tasks jointly while presenting in-distribution and OoD data in the same mix as in \cref{tab:res_cifar10} to the classifier. To this end, we apply gradient boosting to the uncertainty scores provided by the respective methods to predict TP, FP and OoD.
In more detail, we use the following uncertainty scores:
\begin{itemize} \setlength\itemsep{-0.2em}
    \item Ours: OoD uncertainty $C(o|x)$ and entropy $H(x)$ of the estimated class probabilites
    \item Ours with MC dropout: $C(o|x)$, $H(x)$ and the standard deviations of $\hat{p}(y|x)$ summed over all $y=1,\ldots,n$ under MC dropout for $50$ forward passes.
    \item One-vs-All Baseline: Same as ``Ours''.
    \item Max softmax: Maximum softmax probability.
    \item Entropy: Entropy over estimated class probabilities.
    \item Bayes-by-Backprop: For $50$ samples from the posterior we compute $a=\frac{1}{K}\sum_1^K H(\hat{p}(y|x))$ (aleatoric uncertainty) and $b=H(\frac{1}{K}\sum_1^K \hat{p}(y|x))-a$ (epistemic uncertainty) as in \cite{KendallG17}.
    \item MC-Dropout: Entropy of estimated class probabilities averaged over $50$ forward passes, standard deviation of the class probabilities summed over all $y=1,\ldots,n$.
    \item Deep-Ensembles: Entropy of estimated class probabilites averaged over $5$ ensemble members, standard deviation of the class probabilities summed over all $y=1,\ldots,n$.
    \item Confident Classifier: Entropy over estimated class probabilities.
    \item GEN: Entropy over estimated class probabilities resulting of the estimated evidence of the Dirichlet distribution.
\end{itemize}
The corresponding results are given in \cref{tab:res_boosting} in terms of class-wise accuracy. While the right-hand half of \cref{tab:res_boosting} presents results for gradient boosting applied to the uncertainty scores of each method, aiming to predict TP, FP and OoD, the left half of the table shows analogous results while additionally using the estimated class probabilities $\hat{p}(y|x)$ as inputs for gradient boosting. We do so for the sake of accounting for other possible transformations of $\hat{p}(y|x)$ that are not explicitly constructed. The main observations are that our method outperforms the other GAN-based methods and that our method including dropout achieves the overall best performance. It can be observed that the Entropy Oracle performs very strong while using only a single uncertainty score. At a second glance, this is not surprising since a FP mostly involves the confusion of two classes while training the DNN to output maximal entropy on OoD examples is likely to result in the confusion of up to five classes, therefore yielding different entropy levels. Also in the left part as well as the right part of the table, our method including dropout is fairly close the best oracle, which is the entropy oracle. Apart from the oracle, in both studies including and excluding the estimated class probabilities $\hat{p}(y|x)$, our method including MC-dropout outperforms all other methods. However, reviewing the result in an absolute sense, there still remains plenty of room for improvement.

\section{Detailed Results on Individual OoD Datasets} \label{app:ood_datasets}

\begin{table}[t]
    \centering
    \footnotesize
    \setlength{\tabcolsep}{1pt}
    \caption{An OoD-dataset-wise breakdown of the results given in \cref{tab:res_mnist}.}
    \label{tab:res_mnist_breakdown}
    \scalebox{\resultstablescaling}{
    \begin{tabular}{l|cccc||cccc}
        Method & AUROC $\uparrow$ & AUPR-In $\uparrow$ & AUPR-Out $\uparrow$ & \thead{FPR@\\95\% TPR} $\downarrow$ & AUROC $\uparrow$ & AUPR-In $\uparrow$ & AUPR-Out $\uparrow$ & \thead{FPR@\\95\% TPR} $\downarrow$\\\hline
        & \multicolumn{4}{c||}{MNIST 0-4 vs. MNIST 5-9} & \multicolumn{4}{c}{MNIST 0-4 vs. EMNIST-Letters}\\
        Ours                & $92.34\,(2.31)$ & $93.17\,(2.78)$ & $89.96\,(2.23)$ & $38.68\,(7.09)$ & $96.44\,(0.76)$ & $86.32\,(2.69)$ & $99.11\,(0.20)$ & $13.80\,(3.06)$\\
        Ours + Dropout      & $\bm{95.50\,(0.53)}$ & $\bm{96.22\,(0.59)}$ & $93.66\,(0.83)$ & $24.46\,(2.20)$ & $96.71\,(0.55)$ & $86.99\,(2.49)$ & $99.18\,(0.14)$ & $12.49\,(1.63)$\\\hline
        One-vs-All Baseline & $93.65\,(0.97)$ & $92.32\,(1.36)$ & $94.36\,(0.84)$ & $20.06\,(2.98)$ & $91.32\,(0.31)$ & $70.45\,(1.43)$ & $97.69\,(0.10)$ & $29.30\,(1.12)$\\
        Max.\ Softmax        & $92.77\,(0.55)$ & $91.46\,(1.05)$ & $93.58\,(0.26)$ & $22.17\,(1.04)$ & $91.63\,(0.28)$ & $73.52\,(1.02)$ & $97.72\,(0.07)$ & $29.31\,(0.88)$\\
        Entropy             & $92.80\,(0.54)$ & $91.20\,(1.27)$ & $93.63\,(0.24)$ & $22.11\,(1.04)$ & $91.68\,(0.27)$ & $73.51\,(1.14)$ & $97.73\,(0.07)$ & $29.23\,(0.91)$\\
        Bayes-by-Backprop   & $93.73\,(0.99)$ & $92.74\,(1.58)$ & $93.83\,(0.69)$ & $22.17\,(1.40)$ & $90.59\,(0.80)$ & $71.60\,(2.22)$ & $97.34\,(0.23)$ & $34.28\,(1.54)$\\
        MC-Dropout          & $94.32\,(1.10)$ & $93.05\,(1.79)$ & $\bm{95.19\,(0.70)}$ & $17.27\,(1.60)$ & $92.80\,(0.37)$ & $76.85\,(1.53)$ & $98.09\,(0.09)$ & $26.72\,(1.00)$\\
        Deep-Ensembles      & $94.20\,(0.24)$ & $92.82\,(0.18)$ & $95.00\,(0.28)$ & $\bm{16.89\,(1.09)}$ & $93.08\,(0.10)$ & $77.32\,(0.73)$ & $98.16\,(0.02)$ & $24.65\,(0.37)$\\
        Confident Classifier& $95.33\,(0.74)$ & $95.43\,(1.00)$ & $94.95\,(0.64)$ & $19.74\,(1.46)$ & $94.60\,(0.41)$ & $81.71\,(1.56)$ & $98.53\,(0.11)$ & $22.28\,(1.87)$\\
        GEN                 & $88.91\,(2.64)$ & $86.09\,(4.55)$ & $88.15\,(2.20)$ & $40.49\,(5.53)$ & $\bm{99.27\,(0.30)}$ & $\bm{97.05\,(1.22)}$ & $\bm{99.82\,(0.07)}$ & $\bm{3.61\,(1.48)}$\\\hline
        Entropy Oracle      & $99.76\,(0.04)$ & $99.77\,(0.03)$ & $99.74\,(0.04)$ & $0.93\,(0.12)$ & $99.84\,(0.04)$ & $99.43\,(0.12)$ & $99.96\,(0.01)$ & $0.73\,(0.21)$\\
        One-vs-All Oracle   & $99.65\,(0.04)$ & $99.66\,(0.03)$ & $99.65\,(0.04)$ & $1.21\,(0.13)$ & $99.87\,(0.01)$ & $99.49\,(0.05)$ & $99.97\,(0.00)$ & $0.59\,(0.10)$\\\hline\hline
        & \multicolumn{4}{c||}{MNIST 0-4 vs. Omniglot} & \multicolumn{4}{c}{MNIST 0-4 vs. Fashion-MNIST}\\
        Ours                & $96.70\,(0.52)$ & $94.25\,(1.07)$ & $98.32\,(0.26)$ & $18.00\,(3.26)$ & $99.36\,(0.19)$ & $99.04\,(0.32)$ & $99.62\,(0.10)$ & $2.19\,(1.36)$\\
        Ours + Dropout      & $98.45\,(0.24)$ & $96.94\,(0.37)$ & $99.25\,(0.14)$ & $5.95\,(1.08)$ & $99.65\,(0.26)$ & $99.52\,(0.35)$ & $99.78\,(0.17)$ & $0.68\,(0.98)$\\\hline
        One-vs-All Baseline & $98.75\,(0.11)$ & $97.60\,(0.22)$ & $99.41\,(0.05)$ & $4.75\,(0.74)$ & $99.39\,(0.12)$ & $99.15\,(0.16)$ & $99.59\,(0.10)$ & $1.60\,(0.53)$\\
        Max.\ Softmax        & $98.54\,(0.06)$ & $97.10\,(0.24)$ & $99.31\,(0.02)$ & $5.52\,(0.50)$ & $99.29\,(0.23)$ & $99.03\,(0.36)$ & $99.53\,(0.15)$ & $1.79\,(1.34)$\\
        Entropy             & $98.59\,(0.06)$ & $97.15\,(0.24)$ & $99.35\,(0.01)$ & $5.36\,(0.50)$ & $99.35\,(0.22)$ & $99.06\,(0.36)$ & $99.59\,(0.15)$ & $1.74\,(1.31)$\\
        Bayes-by-Backprop   & $96.82\,(0.16)$ & $95.22\,(0.30)$ & $97.46\,(0.23)$ & $14.86\,(0.56)$ & $97.93\,(0.18)$ & $97.64\,(0.24)$ & $98.26\,(0.20)$ & $8.02\,(1.30)$\\
        MC-Dropout          & $\bm{98.96\,(0.09)}$ & $97.56\,(0.27)$ & $\bm{99.58\,(0.03)}$ & $3.85\,(0.43)$ & $99.69\,(0.05)$ & $99.52\,(0.08)$ & $99.83\,(0.03)$ & $0.75\,(0.21)$\\
        Deep-Ensembles      & $98.93\,(0.04)$ & $\bm{97.77\,(0.07)}$ & $99.52\,(0.02)$ & $\bm{3.76\,(0.20)}$ & $99.59\,(0.09)$ & $99.38\,(0.15)$ & $99.76\,(0.05)$ & $1.01\,(0.63)$\\
        Confident Classifier& $98.35\,(0.29)$ & $96.22\,(0.83)$ & $99.29\,(0.11)$ & $6.78\,(1.06)$ & $\bm{99.95\,(0.01)}$ & $\bm{99.91\,(0.03)}$ & $\bm{99.97\,(0.01)}$ & $\bm{0.06\,(0.02)}$\\
        GEN                 & $91.03\,(3.53)$ & $78.34\,(9.87)$ & $94.92\,(1.67)$ & $36.96\,(9.71)$ & $99.92\,(0.02)$ & $99.84\,(0.04)$ & $99.96\,(0.01)$ & $0.23\,(0.11)$\\\hline
        Entropy Oracle      & $99.68\,(0.09)$ & $99.25\,(0.21)$ & $99.87\,(0.04)$ & $1.26\,(0.35)$ & $100.00\,(0.00)$ & $100.00\,(0.00)$ & $100.00\,(0.00)$ & $0.00\,(0.00)$\\
        One-vs-All Oracle   & $99.68\,(0.05)$ & $99.18\,(0.13)$ & $99.88\,(0.02)$ & $1.22\,(0.17)$ & $100.00\,(0.00)$ & $100.00\,(0.00)$ & $100.00\,(0.00)$ & $0.00\,(0.00)$\\\hline\hline
        & \multicolumn{4}{c||}{MNIST 0-4 vs. SVHN} & \multicolumn{4}{c}{MNIST 0-4 vs. CIFAR10}\\
        Ours                & $99.83\,(0.09)$ & $99.50\,(0.21)$ & $99.96\,(0.02)$ & $0.17\,(0.06)$ & $99.84\,(0.08)$ & $99.76\,(0.11)$ & $99.91\,(0.05)$ & $0.21\,(0.12)$\\
        Ours + Dropout      & $99.80\,(0.22)$ & $99.47\,(0.53)$ & $99.94\,(0.07)$ & $0.32\,(0.40)$ & $99.84\,(0.19)$ & $99.78\,(0.26)$ & $99.90\,(0.12)$ & $0.28\,(0.47)$\\\hline
        One-vs-All Baseline & $99.76\,(0.07)$ & $99.30\,(0.17)$ & $99.93\,(0.04)$ & $0.36\,(0.12)$ & $99.55\,(0.16)$ & $99.39\,(0.21)$ & $99.69\,(0.14)$ & $0.75\,(0.45)$\\
        Max.\ Softmax        & $99.68\,(0.12)$ & $99.14\,(0.24)$ & $99.92\,(0.04)$ & $0.39\,(0.12)$ & $99.54\,(0.13)$ & $99.39\,(0.17)$ & $99.69\,(0.10)$ & $0.59\,(0.30)$\\
        Entropy             & $99.75\,(0.11)$ & $99.22\,(0.23)$ & $99.94\,(0.03)$ & $0.37\,(0.11)$ & $99.61\,(0.13)$ & $99.45\,(0.17)$ & $99.76\,(0.09)$ & $0.56\,(0.29)$\\
        Bayes-by-Backprop   & $97.20\,(0.32)$ & $95.75\,(0.34)$ & $98.77\,(0.24)$ & $11.46\,(3.55)$ & $97.52\,(0.40)$ & $97.48\,(0.35)$ & $97.11\,(0.73)$ & $7.44\,(2.13)$\\
        MC-Dropout          & $99.94\,(0.01)$ & $99.75\,(0.05)$ & $99.99\,(0.00)$ & $0.15\,(0.03)$ & $99.92\,(0.03)$ & $99.87\,(0.05)$ & $99.96\,(0.02)$ & $0.09\,(0.05)$\\
        Deep-Ensembles      & $99.89\,(0.01)$ & $99.55\,(0.05)$ & $99.98\,(0.00)$ & $0.25\,(0.06)$ & $99.82\,(0.05)$ & $99.73\,(0.07)$ & $99.90\,(0.03)$ & $0.20\,(0.09)$\\
        Confident Classifier& $\bm{100.00\,(0.00)}$ & $\bm{100.00\,(0.00)}$ & $\bm{100.00\,(0.00)}$ & $\bm{0.00\,(0.00)}$ & $\bm{100.00\,(0.00)}$ & $\bm{100.00\,(0.00)}$ & $\bm{100.00\,(0.00)}$ & $\bm{0.00\,(0.00)}$\\
        GEN                 & $\bm{100.00\,(0.00)}$ & $\bm{100.00\,(0.00)}$ & $\bm{100.00\,(0.00)}$ & $\bm{0.00\,(0.00)}$ & $\bm{100.00\,(0.00)}$ & $\bm{100.00\,(0.01)}$ & $\bm{100.00\,(0.00)}$ & $0.01\,(0.01)$\\\hline
        Entropy Oracle      & $100.00\,(0.00)$ & $100.00\,(0.00)$ & $100.00\,(0.00)$ & $0.00\,(0.00)$ & $100.00\,(0.00)$ & $100.00\,(0.00)$ & $100.00\,(0.00)$ & $0.00\,(0.00)$\\
        One-vs-All Oracle   & $100.00\,(0.00)$ & $100.00\,(0.00)$ & $100.00\,(0.00)$ & $0.00\,(0.00)$ & $100.00\,(0.00)$ & $100.00\,(0.00)$ & $100.00\,(0.00)$ & $0.00\,(0.00)$\\
    \end{tabular}
    }
\end{table}

\begin{table}[t]
    \centering
    \footnotesize
    \setlength{\tabcolsep}{1pt}
    \caption{An OoD-dataset-wise breakdown of the results given in \cref{tab:res_cifar10}.}
    \label{tab:res_cifar10_breakdown}
    \scalebox{\resultstablescaling}{
    \begin{tabular}{l|cccc||cccc}
        Method & AUROC $\uparrow$ & AUPR-In $\uparrow$ & AUPR-Out $\uparrow$ & \thead{FPR@\\95\% TPR} $\downarrow$ & AUROC $\uparrow$ & AUPR-In $\uparrow$ & AUPR-Out $\uparrow$ & \thead{FPR@\\95\% TPR} $\downarrow$\\\hline
        & \multicolumn{4}{c||}{CIFAR10 0-4 vs. CIFAR10 5-9} & \multicolumn{4}{c}{CIFAR10 0-4 vs. LSUN}\\
        Ours                & $65.94\,(0.29)$ & $\bm{72.26\,(0.30)}$ & $64.54\,(0.44)$ & $86.94\,(1.12)$ & $76.45\,(0.52)$ & $69.35\,(0.56)$ & $84.35\,(0.42)$ & $80.50\,(1.31)$\\
        Ours + Dropout      & $\bm{71.31\,(0.49)}$ & $71.98\,(0.35)$ & $\bm{68.45\,(0.59)}$ & $\bm{84.07\,(0.76)}$ & $\bm{82.52\,(0.72)}$ & $\bm{76.03\,(0.80)}$ & $\bm{88.01\,(0.53)}$ & $\bm{72.93\,(1.14)}$\\\hline
        One-vs-All Baseline & $65.31\,(0.67)$ & $66.28\,(0.76)$ & $62.99\,(0.44)$ & $88.15\,(0.52)$ & $74.94\,(0.79)$ & $66.33\,(1.06)$ & $82.99\,(0.61)$ & $82.50\,(1.03)$\\
        Max. Softmax        & $64.45\,(0.55)$ & $65.94\,(0.80)$ & $60.77\,(0.55)$ & $90.73\,(0.33)$ & $72.84\,(0.59)$ & $63.69\,(0.93)$ & $80.70\,(0.45)$ & $87.11\,(0.56)$\\
        Entropy             & $64.64\,(0.53)$ & $65.82\,(0.69)$ & $61.36\,(0.50)$ & $89.50\,(0.69)$ & $73.33\,(0.51)$ & $63.95\,(0.90)$ & $81.62\,(0.39)$ & $83.58\,(0.56)$\\
        Bayes-by-Backprop   & $66.78\,(0.34)$ & $67.16\,(1.07)$ & $63.22\,(0.50)$ & $88.50\,(0.62)$ & $75.31\,(0.65)$ & $66.79\,(1.13)$ & $82.75\,(0.68)$ & $82.64\,(1.16)$\\
        MC-Dropout          & $63.52\,(0.33)$ & $64.11\,(0.38)$ & $60.82\,(0.27)$ & $90.01\,(0.39)$ & $77.04\,(0.22)$ & $70.27\,(0.41)$ & $83.75\,(0.12)$ & $81.53\,(0.59)$\\
        Deep-Ensembles      & $66.85\,(0.38)$ & $67.64\,(0.57)$ & $64.07\,(0.31)$ & $87.59\,(0.42)$ & $78.03\,(0.24)$ & $69.56\,(0.47)$ & $85.49\,(0.20)$ & $77.07\,(0.66)$\\
        Confident Classifier& $65.58\,(0.13)$ & $66.94\,(0.18)$ & $62.60\,(0.29)$ & $88.46\,(0.64)$ & $75.25\,(0.33)$ & $67.26\,(0.41)$ & $82.97\,(0.36)$ & $81.66\,(0.99)$\\
        GEN                 & $65.56\,(0.50)$ & $66.67\,(0.51)$ & $61.90\,(0.67)$ & $89.09\,(1.17)$ & $75.82\,(1.22)$ & $67.74\,(1.65)$ & $82.66\,(0.75)$ & $83.19\,(1.91)$\\\hline
        Entropy Oracle      & $73.21\,(0.52)$ & $75.00\,(0.41)$ & $70.70\,(0.73)$ & $81.41\,(0.88)$ & $98.27\,(0.09)$ & $96.81\,(0.14)$ & $99.11\,(0.06)$ & $7.91\,(0.45)$\\
        One-vs-All Oracle   & $69.45\,(0.73)$ & $69.58\,(0.66)$ & $67.38\,(0.80)$ & $84.92\,(0.92)$ & $95.92\,(0.38)$ & $92.38\,(0.65)$ & $97.93\,(0.22)$ & $19.36\,(2.03)$\\\hline\hline
        & \multicolumn{4}{c||}{CIFAR10 0-4 vs. SVHN} & \multicolumn{4}{c}{CIFAR10 0-4 vs. Fashion-MNIST}\\
        Ours                & $98.50\,(0.26)$ & $92.05\,(1.07)$ & $99.72\,(0.05)$ & $5.99\,(1.27)$ & $78.78\,(0.88)$ & $72.41\,(0.85)$ & $85.12\,(0.92)$ & $81.09\,(2.84)$\\
        Ours + Dropout      & $\bm{98.93\,(0.12)}$ & $\bm{94.40\,(0.62)}$ & $\bm{99.80\,(0.02)}$ & $\bm{5.09\,(0.44)}$ & $\bm{84.75\,(1.32)}$ & $\bm{80.48\,(1.45)}$ & $\bm{88.20\,(1.38)}$ & $\bm{76.11\,(5.07)}$\\\hline
        One-vs-All Baseline & $70.07\,(4.23)$ & $45.76\,(4.99)$ & $89.65\,(1.92)$ & $91.83\,(3.14)$ & $73.94\,(1.31)$ & $68.28\,(1.45)$ & $80.23\,(1.31)$ & $89.25\,(1.33)$\\
        Max. Softmax        & $70.96\,(1.87)$ & $44.63\,(2.22)$ & $90.72\,(0.66)$ & $88.76\,(1.38)$ & $73.90\,(1.18)$ & $67.14\,(1.35)$ & $80.38\,(0.97)$ & $88.48\,(0.82)$\\
        Entropy             & $71.23\,(1.94)$ & $44.81\,(2.17)$ & $90.87\,(0.72)$ & $87.04\,(2.22)$ & $73.93\,(1.26)$ & $67.18\,(1.38)$ & $80.30\,(1.11)$ & $88.47\,(1.71)$\\
        Bayes-by-Backprop   & $76.21\,(0.58)$ & $50.10\,(1.76)$ & $92.92\,(0.24)$ & $80.85\,(1.42)$ & $74.51\,(1.34)$ & $68.52\,(1.75)$ & $80.67\,(1.13)$ & $88.11\,(1.87)$\\
        MC-Dropout          & $76.73\,(2.77)$ & $58.97\,(4.10)$ & $92.23\,(0.86)$ & $84.98\,(1.87)$ & $81.85\,(0.75)$ & $77.62\,(0.79)$ & $85.83\,(0.62)$ & $80.75\,(1.85)$\\
        Deep-Ensembles      & $72.02\,(1.11)$ & $45.95\,(2.24)$ & $90.96\,(0.30)$ & $88.13\,(0.61)$ & $72.82\,(1.20)$ & $66.38\,(2.02)$ & $79.25\,(0.63)$ & $89.97\,(0.84)$\\
        Confident Classifier& $73.60\,(0.61)$ & $48.68\,(1.08)$ & $91.71\,(0.23)$ & $85.45\,(1.08)$ & $74.47\,(0.52)$ & $68.40\,(0.77)$ & $80.85\,(0.35)$ & $87.47\,(0.60)$\\
        GEN                 & $98.43\,(0.42)$ & $91.36\,(1.82)$ & $99.71\,(0.08)$ & $5.62\,(1.68)$ & $72.44\,(3.94)$ & $66.54\,(4.86)$ & $79.12\,(2.47)$ & $89.38\,(2.57)$\\\hline
        Entropy Oracle      & $96.85\,(0.62)$ & $88.28\,(1.42)$ & $99.29\,(0.21)$ & $14.26\,(2.88)$ & $97.89\,(0.16)$ & $96.65\,(0.28)$ & $98.71\,(0.10)$ & $9.50\,(0.61)$\\
        One-vs-All Oracle   & $89.47\,(1.47)$ & $70.44\,(2.54)$ & $97.36\,(0.51)$ & $48.76\,(7.42)$ & $96.63\,(0.40)$ & $94.34\,(0.71)$ & $98.06\,(0.22)$ & $17.37\,(2.48)$\\\hline\hline
        & \multicolumn{4}{c||}{CIFAR10 0-4 vs. MNIST}\\
        Ours                & $83.24\,(3.79)$ & $75.49\,(4.07)$ & $90.43\,(2.76)$ & $58.75\,(13.88)$\\
        Ours + Dropout      & $86.64\,(1.18)$ & $\bm{82.03\,(1.28)}$ & $91.49\,(1.02)$ & $61.32\,(5.75)$\\\cline{1-5}
        One-vs-All Baseline & $78.65\,(0.75)$ & $75.69\,(0.57)$ & $83.12\,(1.09)$ & $86.70\,(2.27)$\\
        Max. Softmax        & $78.89\,(1.79)$ & $74.29\,(2.40)$ & $84.41\,(1.37)$ & $83.07\,(2.10)$\\
        Entropy             & $79.59\,(1.81)$ & $74.66\,(2.38)$ & $85.33\,(1.55)$ & $77.57\,(3.12)$\\
        Bayes-by-Backprop   & $71.46\,(4.50)$ & $62.19\,(6.23)$ & $79.41\,(3.14)$ & $87.01\,(3.35)$\\
        MC-Dropout          & $82.98\,(1.21)$ & $79.05\,(1.68)$ & $87.44\,(0.83)$ & $73.99\,(1.97)$\\
        Deep-Ensembles      & $81.33\,(1.54)$ & $76.97\,(1.67)$ & $85.98\,(1.26)$ & $78.85\,(3.47)$\\
        Confident Classifier& $73.41\,(2.43)$ & $64.90\,(4.16)$ & $81.47\,(1.62)$ & $83.21\,(2.58)$\\
        GEN                 & $\bm{87.63\,(3.97)}$ & $80.69\,(5.53)$ & $\bm{93.31\,(2.29)}$ & $\bm{45.16\,(11.84)}$\\\cline{1-5}
        Entropy Oracle      & $97.59\,(0.26)$ & $97.10\,(0.30)$ & $97.82\,(0.30)$ & $10.16\,(2.03)$\\
        One-vs-All Oracle   & $97.56\,(0.20)$ & $96.52\,(0.40)$ & $98.36\,(0.07)$ & $13.25\,(1.95)$\\
    \end{tabular}
    }
\end{table}

\begin{table}[t]
    \centering
    \footnotesize
    \setlength{\tabcolsep}{1pt}
    \caption{An OoD-dataset-wise breakdown of the results given in \cref{tab:res_cifar100}.}
    \label{tab:res_cifar100_breakdown}
    \scalebox{\resultstablescaling}{
    \begin{tabular}{l|cccc||cccc}
        Method & AUROC $\uparrow$ & AUPR-In $\uparrow$ & AUPR-Out $\uparrow$ & \thead{FPR@\\95\% TPR} $\downarrow$ & AUROC $\uparrow$ & AUPR-In $\uparrow$ & AUPR-Out $\uparrow$ & \thead{FPR@\\95\% TPR} $\downarrow$\\\hline
        & \multicolumn{4}{c||}{CIFAR100 0-49 vs. CIFAR100 50-99} & \multicolumn{4}{c}{CIFAR100 0-49 vs. LSUN}\\
        Ours                 & $64.52\,(0.17)$ & $65.00\,(0.35)$ & $61.84\,(0.21)$ & $89.57\,(0.47)$ & $65.30\,(0.56)$ & $52.94\,(0.30)$ & $75.99\,(0.42)$ & $90.36\,(0.61)$\\
        Ours + Dropout       & $\bm{66.97\,(0.30)}$ & $65.14\,(0.37)$ & $\bm{64.41\,(0.41)}$ & $\bm{87.74\,(0.48)}$ & $68.60\,(0.62)$ & $56.80\,(0.81)$ & $78.09\,(0.46)$ & $88.77\,(0.65)$\\\hline
        One-vs-All Baseline  & $61.62\,(0.31)$ & $61.30\,(0.42)$ & $58.80\,(0.34)$ & $91.49\,(0.37)$ & $64.09\,(1.09)$ & $51.10\,(1.14)$ & $75.24\,(0.84)$ & $90.81\,(0.92)$\\
        Max. Softmax         & $62.43\,(0.72)$ & $62.87\,(1.41)$ & $59.39\,(0.61)$ & $90.92\,(0.48)$ & $65.21\,(1.34)$ & $53.47\,(1.45)$ & $76.01\,(1.09)$ & $89.96\,(1.00)$\\
        Entropy              & $63.53\,(0.62)$ & $63.45\,(1.36)$ & $60.59\,(0.61)$ & $90.23\,(0.76)$ & $66.62\,(1.51)$ & $54.37\,(1.65)$ & $77.22\,(1.16)$ & $88.76\,(1.16)$\\
        Bayes-by-Backprop    & $64.16\,(0.36)$ & $64.56\,(0.61)$ & $60.64\,(0.41)$ & $90.44\,(0.47)$ & $67.02\,(0.84)$ & $55.60\,(1.09)$ & $76.79\,(0.62)$ & $90.05\,(0.57)$\\
        MC-Dropout           & $62.97\,(0.22)$ & $62.21\,(0.41)$ & $60.02\,(0.32)$ & $90.34\,(0.37)$ & $67.19\,(1.00)$ & $56.12\,(0.89)$ & $77.92\,(0.92)$ & $87.38\,(1.22)$\\
        Deep-Ensembles       & $66.95\,(0.20)$ & $\bm{65.94\,(0.39)}$ & $63.82\,(0.10)$ & $87.90\,(0.47)$ & $\bm{71.34\,(0.64)}$ & $\bm{59.76\,(0.85)}$ & $\bm{80.73\,(0.62)}$ & $\bm{84.60\,(1.38)}$\\
        Confident Classifier & $62.39\,(0.67)$ & $62.31\,(0.12)$ & $59.78\,(0.71)$ & $90.10\,(0.54)$ & $64.24\,(0.83)$ & $51.61\,(0.74)$ & $75.73\,(0.72)$ & $89.53\,(0.94)$\\
        GEN                  & $62.66\,(0.40)$ & $61.82\,(0.50)$ & $60.39\,(0.48)$ & $89.89\,(0.69)$ & $62.59\,(1.01)$ & $51.97\,(0.51)$ & $72.87\,(1.06)$ & $93.96\,(1.29)$\\\hline
        Entropy Oracle       & $64.42\,(0.31)$ & $65.50\,(0.31)$ & $61.35\,(0.37)$ & $89.87\,(0.86)$ & $70.00\,(0.33)$ & $58.01\,(0.24)$ & $79.95\,(0.34)$ & $85.32\,(0.88)$\\
        One-vs-All Oracle    & $67.55\,(1.07)$ & $66.49\,(1.24)$ & $65.17\,(1.02)$ & $86.70\,(0.53)$ & $78.10\,(0.99)$ & $64.04\,(1.12)$ & $86.95\,(0.92)$ & $70.25\,(2.55)$\\\hline\hline
        & \multicolumn{4}{c||}{CIFAR100 0-49 vs. SVHN} & \multicolumn{4}{c}{CIFAR100 0-49 vs. Fashion-MNIST}\\
        Ours                 & $\bm{96.47\,(1.26)}$ & $\bm{81.59\,(5.55)}$ & $\bm{99.25\,(0.26)}$ & $\bm{11.95\,(3.93)}$ & $62.98\,(1.71)$ & $51.89\,(2.88)$ & $73.95\,(1.21)$ & $92.37\,(1.04)$\\
        Ours + Dropout       & $95.63\,(1.27)$ & $80.03\,(4.04)$ & $99.17\,(0.26)$ & $15.47\,(4.21)$ & $64.68\,(2.02)$ & $57.46\,(1.99)$ & $74.57\,(1.48)$ & $91.50\,(1.15)$\\\hline
        One-vs-All Baseline  & $60.63\,(3.90)$ & $26.37\,(6.04)$ & $87.22\,(1.37)$ & $93.11\,(1.92)$ & $59.80\,(2.67)$ & $51.35\,(3.14)$ & $70.68\,(1.29)$ & $95.13\,(0.65)$\\
        Max. Softmax         & $66.32\,(2.15)$ & $37.61\,(3.13)$ & $89.38\,(0.78)$ & $89.73\,(1.60)$ & $68.46\,(0.56)$ & $60.24\,(0.85)$ & $77.66\,(0.47)$ & $89.36\,(0.46)$\\
        Entropy              & $68.09\,(2.54)$ & $38.90\,(3.51)$ & $89.84\,(0.88)$ & $89.65\,(1.44)$ & $68.95\,(0.48)$ & $60.59\,(0.68)$ & $78.35\,(0.45)$ & $87.66\,(0.75)$\\
        Bayes-by-Backprop    & $72.62\,(1.07)$ & $42.79\,(1.62)$ & $92.21\,(0.50)$ & $81.59\,(2.61)$ & $66.22\,(2.02)$ & $57.55\,(3.02)$ & $75.36\,(1.32)$ & $91.91\,(1.20)$\\
        MC-Dropout           & $65.65\,(2.51)$ & $35.78\,(3.03)$ & $88.94\,(1.09)$ & $91.07\,(1.95)$ & $\bm{70.50\,(1.49)}$ & $\bm{63.75\,(1.52)}$ & $\bm{78.70\,(1.54)}$ & $88.43\,(2.58)$\\
        Deep-Ensembles       & $75.01\,(1.06)$ & $49.77\,(2.19)$ & $91.86\,(0.46)$ & $87.10\,(1.26)$ & $67.42\,(1.14)$ & $60.17\,(1.35)$ & $77.95\,(0.83)$ & $\bm{85.46\,(1.12)}$\\
        Confident Classifier & $68.73\,(0.49)$ & $39.41\,(0.95)$ & $90.33\,(0.17)$ & $87.73\,(0.71)$ & $67.04\,(0.89)$ & $56.88\,(0.96)$ & $78.16\,(0.38)$ & $86.27\,(1.46)$\\
        GEN                  & $92.83\,(2.72)$ & $71.54\,(8.20)$ & $98.58\,(0.56)$ & $23.57\,(7.35)$ & $59.12\,(6.18)$ & $49.44\,(5.14)$ & $70.39\,(4.78)$ & $94.58\,(3.06)$\\\hline
        Entropy Oracle       & $86.97\,(0.75)$ & $63.53\,(1.22)$ & $96.97\,(0.21)$ & $50.09\,(2.55)$ & $97.47\,(0.45)$ & $94.63\,(0.85)$ & $98.83\,(0.22)$ & $10.94\,(1.96)$\\
        One-vs-All Oracle    & $98.17\,(0.22)$ & $90.82\,(0.89)$ & $99.65\,(0.04)$ & $8.11\,(0.98)$ & $99.62\,(0.05)$ & $99.13\,(0.12)$ & $99.83\,(0.02)$ & $1.65\,(0.23)$\\\hline\hline
        & \multicolumn{4}{c||}{CIFAR100 0-49 vs. MNIST}\\
        Ours                 & $77.28\,(5.81)$ & $70.28\,(6.06)$ & $84.38\,(4.72)$ & $78.32\,(10.93)$\\
        Ours + Dropout       & $77.13\,(3.87)$ & $73.64\,(4.41)$ & $81.30\,(2.83)$ & $90.07\,(2.55)$\\\cline{1-5}
        One-vs-All Baseline  & $71.94\,(3.24)$ & $64.45\,(3.29)$ & $78.80\,(2.69)$ & $90.07\,(2.67)$\\
        Max. Softmax         & $75.57\,(4.00)$ & $68.72\,(4.65)$ & $83.09\,(2.95)$ & $81.22\,(5.60)$\\
        Entropy              & $79.14\,(4.38)$ & $72.10\,(5.53)$ & $85.59\,(3.17)$ & $76.56\,(8.39)$\\
        Bayes-by-Backprop    & $71.27\,(2.11)$ & $63.84\,(3.33)$ & $78.08\,(1.57)$ & $91.43\,(1.92)$\\
        MC-Dropout           & $73.41\,(1.90)$ & $66.48\,(3.12)$ & $80.45\,(1.38)$ & $87.39\,(2.49)$\\
        Deep-Ensembles       & $\bm{85.92\,(1.62)}$ & $\bm{82.02\,(1.88)}$ & $\bm{89.77\,(1.50)}$ & $\bm{68.03\,(5.45)}$\\
        Confident Classifier & $77.66\,(1.41)$ & $69.86\,(1.33)$ & $85.11\,(0.82)$ & $76.68\,(1.55)$\\
        GEN                  & $77.88\,(6.84)$ & $71.93\,(7.67)$ & $83.30\,(5.81)$ & $80.26\,(14.55)$\\\cline{1-5}
        Entropy Oracle       & $99.79\,(0.06)$ & $99.56\,(0.12)$ & $99.90\,(0.03)$ & $1.02\,(0.28)$\\
        One-vs-All Oracle    & $99.99\,(0.01)$ & $99.97\,(0.01)$ & $99.99\,(0.00)$ & $0.03\,(0.03)$\\
    \end{tabular}
    }
\end{table}

\begin{table}[t]
    \centering
    \footnotesize
    \setlength{\tabcolsep}{1pt}
    \caption{An OoD-dataset-wise breakdown of the results given in \cref{tab:res_tinyimagenet}.}
    \label{tab:res_tinyimagenet_breakdown}
    \scalebox{\resultstablescaling}{
    \begin{tabular}{l|cccc||cccc}
        Method & AUROC $\uparrow$ & AUPR-In $\uparrow$ & AUPR-Out $\uparrow$ & \thead{FPR@\\95\% TPR} $\downarrow$ & AUROC $\uparrow$ & AUPR-In $\uparrow$ & AUPR-Out $\uparrow$ & \thead{FPR@\\95\% TPR} $\downarrow$\\\hline
        & \multicolumn{4}{c||}{Tiny ImageNet 0-99 vs. Tiny ImageNet 100-199} & \multicolumn{4}{c}{Tiny ImageNet 0-99 vs. SVHN}\\
        Ours                 & $59.16\,(0.52)$ & $61.14\,(0.44)$ & $56.15\,(0.40)$ & $93.10\,(0.35)$ & $98.98\,(0.44)$ & $93.46\,(2.87)$ & $99.81\,(0.08)$ & $3.49\,(1.31)$\\
        Ours + Dropout       & $\bm{62.01\,(0.22)}$ & $\bm{64.70\,(0.27)}$ & $\bm{58.16\,(0.24)}$ & $\bm{91.97\,(0.40)}$ & $\bm{99.39\,(0.41)}$ & $\bm{96.18\,(2.67)}$ & $\bm{99.89\,(0.07)}$ & $\bm{2.19\,(1.36)}$\\\hline
        One-vs-All Baseline  & $58.53\,(0.45)$ & $60.28\,(0.33)$ & $55.59\,(0.43)$ & $93.29\,(0.27)$ & $59.65\,(4.57)$ & $32.66\,(4.11)$ & $85.21\,(1.86)$ & $97.02\,(0.95)$\\
        Max. Softmax         & $58.16\,(0.25)$ & $60.98\,(0.30)$ & $55.48\,(0.28)$ & $93.22\,(0.52)$ & $63.33\,(1.39)$ & $32.44\,(2.05)$ & $88.51\,(0.54)$ & $90.93\,(1.05)$\\
        Entropy              & $58.69\,(0.30)$ & $61.42\,(0.31)$ & $55.68\,(0.12)$ & $93.30\,(0.27)$ & $65.35\,(1.48)$ & $34.25\,(1.80)$ & $88.73\,(0.76)$ & $91.90\,(1.66)$\\
        Bayes-by-Backprop    & $57.99\,(0.47)$ & $60.44\,(0.42)$ & $55.25\,(0.35)$ & $93.18\,(0.60)$ & $74.18\,(1.78)$ & $38.32\,(1.70)$ & $93.23\,(0.81)$ & $73.54\,(5.25)$\\
        MC-Dropout           & $61.12\,(0.40)$ & $64.16\,(0.39)$ & $57.24\,(0.49)$ & $93.11\,(0.62)$ & $63.54\,(4.81)$ & $37.05\,(3.30)$ & $87.00\,(2.21)$ & $95.33\,(1.73)$\\
        Deep-Ensembles       & $60.70\,(0.23)$ & $64.15\,(0.23)$ & $56.84\,(0.24)$ & $93.16\,(0.34)$ & $73.51\,(0.81)$ & $48.67\,(1.31)$ & $90.92\,(0.42)$ & $90.76\,(1.18)$\\
        Confident Classifier & $58.45\,(0.23)$ & $61.26\,(0.31)$ & $55.36\,(0.17)$ & $93.22\,(0.53)$ & $61.86\,(1.66)$ & $31.25\,(1.56)$ & $87.62\,(0.61)$ & $92.79\,(0.76)$\\
        GEN                  & $56.32\,(0.80)$ & $58.82\,(0.48)$ & $53.80\,(0.79)$ & $94.07\,(0.60)$ & $91.40\,(6.55)$ & $71.24\,(16.70)$ & $98.24\,(1.40)$ & $24.80\,(15.25)$\\\hline
        Entropy Oracle       & $58.57\,(0.73)$ & $61.46\,(0.75)$ & $55.49\,(0.65)$ & $93.23\,(0.68)$ & $84.04\,(2.72)$ & $57.38\,(4.13)$ & $96.14\,(0.82)$ & $57.04\,(8.27)$\\
        One-vs-All Oracle    & $60.16\,(0.31)$ & $61.23\,(0.23)$ & $57.46\,(0.35)$ & $91.92\,(0.57)$ & $99.39\,(0.19)$ & $95.80\,(1.31)$ & $99.89\,(0.03)$ & $2.25\,(0.70)$\\\hline\hline
        & \multicolumn{4}{c||}{Tiny ImageNet 0-99 vs. FMNIST} & \multicolumn{4}{c}{Tiny ImageNet 0-99 vs. MNIST}\\
        Ours                 & $55.53\,(4.69)$ & $42.81\,(4.04)$ & $69.36\,(3.95)$ & $93.93\,(3.42)$ & $61.69\,(3.34)$ & $49.40\,(3.00)$ & $73.63\,(3.26)$ & $91.39\,(3.62)$\\
        Ours + Dropout       & $\bm{95.06\,(1.45)}$ & $\bm{89.18\,(2.88)}$ & $\bm{97.75\,(0.68)}$ & $\bm{17.29\,(4.71)}$ & $\bm{99.78\,(0.18)}$ & $\bm{99.52\,(0.39)}$ & $\bm{99.90\,(0.09)}$ & $\bm{1.04\,(0.94)}$\\\hline
        One-vs-All Baseline  & $45.63\,(4.00)$ & $41.60\,(3.30)$ & $59.66\,(2.04)$ & $99.70\,(0.21)$ & $51.48\,(10.97)$ & $45.52\,(10.95)$ & $64.34\,(7.06)$ & $97.71\,(2.51)$\\
        Max. Softmax         & $61.16\,(2.40)$ & $51.69\,(2.62)$ & $71.73\,(1.96)$ & $94.55\,(1.42)$ & $58.91\,(2.05)$ & $48.66\,(2.13)$ & $70.89\,(2.03)$ & $94.19\,(1.70)$\\
        Entropy              & $60.40\,(2.88)$ & $52.07\,(2.77)$ & $69.67\,(2.27)$ & $97.10\,(1.16)$ & $58.77\,(2.97)$ & $49.18\,(2.26)$ & $69.79\,(3.04)$ & $95.64\,(2.23)$\\
        Bayes-by-Backprop    & $61.83\,(2.86)$ & $50.29\,(2.99)$ & $73.37\,(2.38)$ & $91.78\,(2.26)$ & $63.37\,(5.57)$ & $48.91\,(6.19)$ & $76.10\,(4.20)$ & $87.10\,(5.19)$\\
        MC-Dropout           & $56.34\,(6.24)$ & $53.15\,(4.84)$ & $65.90\,(3.83)$ & $98.79\,(0.84)$ & $71.00\,(3.85)$ & $66.16\,(4.66)$ & $76.05\,(2.72)$ & $95.63\,(2.55)$\\
        Deep-Ensembles       & $58.77\,(1.34)$ & $55.05\,(2.29)$ & $66.48\,(0.83)$ & $99.35\,(0.25)$ & $65.32\,(1.95)$ & $59.31\,(1.74)$ & $72.49\,(1.96)$ & $96.40\,(1.83)$\\
        Confident Classifier & $58.59\,(2.13)$ & $49.75\,(1.96)$ & $68.94\,(1.61)$ & $97.04\,(0.84)$ & $57.31\,(5.67)$ & $49.00\,(5.73)$ & $68.20\,(3.76)$ & $96.97\,(1.45)$\\
        GEN                  & $72.10\,(11.72)$ & $60.00\,(14.00)$ & $83.71\,(7.27)$ & $67.91\,(17.06)$ & $93.78\,(4.14)$ & $87.74\,(6.11)$ & $96.64\,(2.87)$ & $22.91\,(18.67)$\\\hline
        Entropy Oracle       & $91.63\,(2.01)$ & $85.61\,(3.03)$ & $95.54\,(1.20)$ & $35.67\,(8.17)$ & $96.29\,(1.14)$ & $93.32\,(2.04)$ & $98.11\,(0.61)$ & $17.98\,(5.19)$\\
        One-vs-All Oracle    & $99.81\,(0.11)$ & $99.50\,(0.27)$ & $99.92\,(0.05)$ & $0.74\,(0.43)$ & $100\,(0.00)$ & $100\,(0.00)$ & $100\,(0.00)$ & $0.00\,(0.00)$
    \end{tabular}
    }
\end{table}

In \cref{sec:experiments} we have presented results on the in-distribution datasets MNIST 0-4, CIFAR10 0-4 and CIFAR100 0-49 versus the entirety of all respective OoD datasets. To give more detailed insights, this appendix section contains comparisons of the respective in-distribution dataset and each of the single corresponding OoD datasets.

The comparison by dataset for the MNIST 0-4 as in-distribution set in the OoD-dataset-wise breakdown given in \cref{tab:res_mnist_breakdown} shows that we achieve mid-tier results compared to the other methods, except for when MNIST 5-9 is considered as the OoD dataset. In that challenging case, we achieve stronger results compared to the other methods.
Examining the results of our OoD experiments with CIFAR10 0-4 being the in-distribution dataset in \cref{tab:res_cifar10_breakdown}, it becomes apparent that we consistently outperform all other methods on each single OoD dataset. The performance gains are particularly pronounced when using the very similar datasets CIFAR10 5-9 and LSUN as OoD datasets. For CIFAR100 0-49 we have a similar situation as for MNIST, where our approach performs particularly well on the difficult task of CIFAR100 0-49 vs. CIFAR100 50-99 while achieving mid-tier results on the other tasks. This consistency supports the finding that our method -- especially in the case of OoD datasets very similar to the in-distribution dataset -- shows the most improvement compared to other methods. This finding might be to some extent attributable to our tight class shielding.

\end{document}